\documentclass[a4paper,10pt]{article}
\usepackage[utf8x]{inputenc}
\usepackage{amssymb,amsmath,amsthm}
\usepackage{hyperref}
\usepackage{xcolor}
\usepackage{bbm}
\usepackage{amssymb}
\usepackage{graphicx}
\usepackage{amsmath}
\usepackage{subfigure}
\usepackage{caption}
\usepackage{capt-of}
\usepackage{amsthm}
\usepackage{multirow}
\usepackage{bbm}
\newcommand{\norm}[1]{\left\lVert#1\right\rVert}
\usepackage[english]{babel}
\usepackage{natbib} 

\usepackage{comment}

\usepackage{authblk}

\newtheorem{prop}{Proposition}
\newtheorem{remark}{Remark}

\newcommand{\xx}{\mathbf{x}}
\newcommand{\yy}{\mathbf{y}}

\newcommand{\R}{\mathbb{R}}
\newcommand{\bset}{\mathcal{X}}
\newcommand{\sset}{\mathcal{S}}
\newcommand{\kin}{k_{\bset}} 
\newcommand{\kout}{k_{\text{H}}} 

\newcommand{\kde}{k_{\text{DE}}}
\newcommand{\kds}{k_{0}}

\newcommand{\sets}{\mathbf{S}}
\newcommand{\rank}{\operatorname{rank}}
\newcommand{\Hkin}{\mathcal{H}_{\kin}}
\newcommand{\Sfin}{\sset_{\text{fin}}(\bset)} 
\newcommand{\emb}{\mathcal{E}}
\newcommand{\embr}{\mathcal{E}_{\rin}}
\newcommand{\rin}{r_{\bset}}
\newcommand{\varin}{\sigma^2_{\bset}}
\newcommand{\sigin}{\sigma_{\bset}}
\newcommand{\psiin}{\psi_{\bset}}
\newcommand{\thetain}{\theta_{\bset}}
\newcommand{\rout}{r_{H}}
\newcommand{\varout}{\sigma^2_H}
\newcommand{\sigout}{\sigma_H}

\newcommand{\thetaout}{\theta_{H}}
\title{
Kernels over Sets of Finite Sets using RKHS Embeddings, with Application to Bayesian (Combinatorial) Optimization 
}

\newcommand*\samethanks[1][\value{footnote}]{\footnotemark[#1]}

\author[1]{Poompol Buathong\thanks{PB and DG contributed equally to this work and are in alphabetical order.}}
\author[2,3]{David Ginsbourger\samethanks}
\author[1,4]{Tipaluck Krityakierne}

\affil[1]{Department of Mathematics, Faculty of Science, Mahidol University, Bangkok, Thailand}
\affil[2]{Uncertainty Quantification and Optimal Design group, Idiap Research Institute, Centre du Parc, Rue Marconi 19, PO Box 592, CH-1920 Martigny, Switzerland.}
\affil[3]{Institute of Mathematical Statistics and Actuarial Science, Department of Mathematics and Statistics, University of Bern, Alpeneggstrasse 22, CH-3012 Bern, Switzerland}
\affil[4]{Centre of Excellence in Mathematics, CHE, Bangkok, Thailand}

\begin{document}
	
\maketitle

\thispagestyle{empty}

\begin{abstract}
	We focus on kernel methods for set-valued inputs and their application to Bayesian set optimization, notably combinatorial optimization. 
	We investigate two classes of set kernels that both rely on Reproducing Kernel Hilbert Space embeddings, namely the ``Double Sum'' (DS) kernels recently considered in Bayesian set optimization, and a class introduced here called ``Deep Embedding'' (DE) kernels that essentially consists in applying a radial kernel on Hilbert space on top of the canonical distance induced by another kernel such as a DS kernel. 
	We establish in particular that while DS kernels typically suffer from a lack of strict positive definiteness, vast subclasses of DE kernels built upon DS kernels do possess this property, enabling in turn combinatorial optimization without requiring to introduce a jitter parameter.  
	Proofs of theoretical results about considered kernels are complemented by a few practicalities regarding hyperparameter fitting. 
	We furthermore demonstrate the applicability of our approach in prediction and optimization tasks, relying both on toy examples and on two test cases from mechanical engineering and hydrogeology, respectively. 
	Experimental results highlight the applicability and compared merits of the considered approaches while opening new perspectives in prediction and sequential design with set inputs.   
\end{abstract}

\section{Introduction}

Kernel methods \citep{Aronszajn1950a, Kimeldorf.Wahba1970, Scholkopf.Smola2002, Saitoh2016} constitute a versatile framework for a variety of tasks in classification \citep{Steinwart2008}, function approximation based on scattered data \citep{Wendland2005}, and probabilistic prediction \citep{rasmussen2006gaussian}.  
One of the outstanding features of Gaussian Process (GP) prediction, in particular, is its usability to design Bayesian Optimization (BO) algorithms \citep{mockus1978application, jones1998efficient, Frazier2018} and further sequential design strategies \citep{Risk2018, Binois, Bect.etal}. While in most usual GP- and BO-related contributions the focus is on continuous problems with vector-valued inputs, there has been a growing interest recently for situations involving discrete and mixed discrete-continuous inputs \citep{Kondor2002, Gramacy2010,  Fortuin, Roustant, Garrido-Merchan2018, Ru2019, Griffiths2019}. Here we focus specifically on kernels dedicated to finite set-valued inputs and their application to GP modelling and BO, notably (but not only) in combinatorial optimization.

A number of prediction and optimization problems from various application domains involve finite set-valued inputs, 
encompassing for instance sensor network design \citep{garnett2010bayesian}, simulation-based investigation of the mechanical behaviour of bi-phasic materials depending on the positions of inclusions \citep{ginsbourger2016design}, inventory system optimization \citep{Salemi2019}, selection of starting centers in clustering algorithms \citep{Kim}, speaker recognition and image texture classification (as mentioned by \cite{Desobry2005}), natural language processing tasks with bags of words \citep{Pappas2017}, or optimal positioning of landmarks in shape analysis \citep{Iwata2012}, to cite a few. 
Yet, the number of available kernel methods for efficiently tackling such problems is still quite moderate, although the topic has gained interest among the machine learning and further research communities in the last few years. 
In particular, early investigations regarding the definition of positive definite kernels on finite sets encompass \citep{Kondor2003, Grauman2007}, and also indirectly \citep{Cuturi.etal2005} where kernels between atomic measures are introduced. 
Kernels on finite sets that have been used in BO include radial kernels with respect to the earth mover's distance \citep[][where the question of their positive definiteness is not discussed]{garnett2010bayesian}, kernels on graphs implicitly defined via precision matrices in the context of Gaussian Markov Random Fields in \citep{Salemi2019}, and the class used in \citep{Kim} and originating in \citep{Haussler1999, Gaertner2002} that we refer to as \textit{Double Sum} (DS) kernels. 
From the combinatorial optimization side, while an approach relying on Bayesian networks was considered already in \citep{Larraiiaga2000}, the topic has recently attracted attention in GP-based BO with respect to set inputs (see for instance \cite{Baptista2018} where the emphasis is not on the employed kernels, and \cite{Oh2019} where graph representations are used), and also in GP-based BO over the latent space of a variational autoencoder \citep{Griffiths2019}.

Our approach here is to leverage the fertile framework of Reproducing Kernel Hilbert Space Embeddings \citep{Berlinet.Thomas-Agnan2004, Smola2007, Sriperumbudur2011, Muandet2017} to 
analyze DS kernels and the introduced \textit{Deep Embedding} (DE) kernels, that consist in chaining radial kernels in Hilbert space with the canonical distance associated with set kernels like DS ones. As we establish, wide classes of DE kernels are strictly positive definite which contrasts with the typical case of DS kernels. 
We present in turn a few additional results pertaining to the parametrization of DE kernels and to related hyperparameter fitting, including geometrical considerations around the choice of hyperparameter bounds. 
Section~\ref{sec2} is mainly dedicated to the exposition and theoretical analysis of the considered classes of kernels, complemented by practicalities regarding hyperparameter fitting. In Section~3, numerical experiments are discussed that compare DS and DE kernels in prediction and optimization tasks, both on analytical and on two application test cases, namely in mechanical engineering with plasticity simulations of a bi-phasic material tackled in \citep{ginsbourger2016design}, and in hydrogeology with an original monitoring well selection problem based on the contaminant source localization test case from \citep{pirot2019contaminant}.

\section{Set Kernels via RKHS Embeddings}
\label{sec2}
\subsection{Notation and Settings}

We focus on positive definite kernels defined over subsets of some base set $\bset$. Depending on the cases, $\bset$ may be finite or infinite. The considered set of subsets of $\bset$, denoted $\sset(\bset)$, may be the whole power set $\mathcal{P}(\bset)$ or a subset thereof, e.g. $\mathcal{S}_{p}(\bset)$ (also traditionally noted $[\bset]^p$ in set theory) the set of $p$-element subsets of $\bset$ (where $p \in \mathbb{N}$, with $p\leq \#\bset$ in case of a finite $\bset$ with cardinality $\#\bset$), or the set of all (non-void) finite subsets of $\bset$ denoted here $\Sfin=\cup_{p\geq 1}\mathcal{S}_{p}(\bset)$. Given a positive definite kernel $\kin$ over $\bset$ and the associated Reproducing Kernel Hilbert Space $\Hkin$, we call here \textit{embedding of $\Sfin$ in $\Hkin$} the mapping 
\begin{equation}
\label{def_embed}
\emb: S \in \Sfin \to 
\frac{1}{\# S} \sum_{\xx\in S} \kin(\xx,\cdot) \in \Hkin.
\end{equation}
Note that this ``set embedding'' coincides with the Kernel Mean Embedding \citep{Muandet2017} in $\Hkin$ of the uniform probability distribution over $S$. 

\subsection{From Linear to Deep Embedding Kernels}

A natural idea to create a positive definite kernel on $\Sfin$ from this embedding is to plainly take: 
\begin{equation}
\label{def_dsk}
\begin{split}
\kds(S,S') 
= \frac{1}{\# S \# S'}  \sum_{\substack{\xx\in S\\ \xx'\in S'}} \kin(\xx,\xx'),
\end{split}
\end{equation}
which is none other than the kernel used in \citep{Kim} and that we refer to here as double sum kernel. As we will see in the next section and in the applications, this positive definite kernel may suffer in some settings from its lack of strict positive definiteness. Yet it appears as a crucial building block in the class of strictly positive definite kernels that we introduce here. The first step is to consider the ``canonical distance'' on  $\Sfin$ induced by the kernel $k_{0}$, namely 
\begin{align}
\label{def_cd}
d_{\emb}(S,S')
=\sqrt{\kds(S,S)+\kds(S',S')-2 \kds(S,S')}.
\end{align}
Coming now to the proposed class of \textit{Deep Embedding} kernels per se, these are obtained by composing what can be called a radial kernel on Hilbert space (See \citep{bachoc2018gaussian3} 
for a reminder) with $d_{\emb}$ above. We hence obtain DE kernels on $\Sfin$ by writing 
\begin{equation}
\label{eq_proposed_k}
\begin{split}
\kde(S,S')
&=\kout \circ d_{\emb} (S,S'), 
\end{split}
\end{equation}
with $\kout: [0,\infty) \to \R$ being such that $(h,h')\in \mathcal{H}^2 \to \kout(||h-h'||_{\mathcal{H}})$ is positive definite for any Hilbert space $(\mathcal{H}, \langle \cdot, \cdot, \rangle_{\mathcal{H}})$.  
We establish next the positive definiteness of such kernels (See \citep{Berg.etal1984, Christmann2010} for similar constructions) and further provide sufficient conditions for their strict positive definiteness on $\Sfin$, a feature that $\kds$ is lacking, as we show too, which may lead to invertibility issues for finite $\bset$, e.g. in combinatorial optimization.

\subsection{Main Theoretical Results}

\begin{prop}
	\label{prop1}
	Let $\bset$ be a set, $\kin$ be a positive definite kernel on $\bset$ with associated reproducing kernel Hilbert space $\Hkin$, 
	and $\Sfin$ be the set of non-empty finite subsets of $\bset$. 
	Let $\emb: S \in \Sfin \mapsto \Hkin$, $\kds: \Sfin \times \Sfin \mapsto \R$, $d_{\emb}: \Sfin \times \Sfin \mapsto [0,\infty)$ be defined by Equations~\ref{def_embed},\ref{def_dsk},\ref{def_cd}, respectively. Then, 
	\begin{description}
		\item[a)] $\kds(S,S')=\langle \emb(S), \emb(S') \rangle_{\Hkin}$ for any $S,S'\in \Sfin$, and $\kds$ is positive definite on $\Sfin$ while $d_{\emb}$ is a pseudometric on $\Sfin$.
	\end{description}
	Let us furthermore introduce for $n \geq 2$ the sets 
	\begin{scriptsize}
		\begin{align*}
		A_n=\Biggl\{ &
		\Biggl(
		\overbrace{\frac{1}{n_1},\ldots,\frac{1}{n_1}}^{(n_1-\ell) \text{ times}},
		\overbrace{\frac{n_2-n_1}{n_1 n_2},\ldots,\frac{n_2-n_1}{n_1 n_2}}^{\ell \text{ times}},
		\overbrace{\frac{-1}{n_2},\ldots,\frac{-1}{n_2}}^{(n_2-\ell) \text{ times}}
		\Biggl), 
		\\
		& n_1, n_2 \geq 1, \ell \geq 0: n_1+n_2+\ell= n \Biggl\} \subset \R^n \ (n\geq 2). 
		\end{align*}
	\end{scriptsize}
	\begin{description}
		\item[b)] Then, the following assertions are equivalent: 
		\begin{description}
			\item[i)] $\kin$ satisfies $\sum_{i=1}^n \sum_{j=1}^n a_i a_j \kin(\xx_i,\xx_j) >0$ for all $n\geq 2$, pairwise distinct $\xx_1,\ldots, \xx_n \in \bset$, and $(a_1,\dots,a_n) \in A_n$. 
			\item[ii)] $\emb$ is injective. 
			\item[iii)] $d_{\emb}$ is a metric on $\Sfin$. 
		\end{description}	
	\end{description}
	In particular, if $\kin$ is strictly positive definite on $\bset$, then all three conditions above are fulfilled. 
\end{prop}

\begin{prop}[Non-strict positive definiteness of double sum kernels]
	\label{th_nonstrict_dsk}
	Let us keep the notation of Proposition~\ref{prop1} and denote furthermore in the case of a finite set $\bset$ with cardinality $c\geq 1$ and elements $\mathbf{X}_c=(\xx_1, \dots,\xx_c)$ by $u: S\in\Sfin \to u(S)=\frac{1}{\# S} (\mathbf{1}_{\xx_i \in S})_{1\leq i \leq c} \in \R^c$ the mapping returning for any nonempty subset of $\bset$ a vector with components $\frac{1}{\# S}$ or $0$ depending whether $\xx_i\in S$ or not. 
	Then we have:  
	\begin{description}
		\item[a)] For $\bset$ finite, 
		for any $S,S' \in \Sfin$,
		\begin{equation}
		\label{compact_dsk}
		\kds(S,S')=u(S)^{T} \kin(\mathbf{X}_c) u(S').
		\end{equation}
		Consequently, for $q\geq 1$ and $\sets=(S_1,\dots,S_q)\in \sset^q$, the covariance matrix $\kds(\sets)$ associated with $\kin$ and $\sets$ can be compactly written as
		\begin{equation}
		\label{compact_dsk_matrix}
		\kds(\sets)=U(\sets)^{T}\kin(\mathbf{X}_c) U(\sets),
		\end{equation}	
		with the notation $U(\sets)=[u(S_1),\dots,u(S_q)]$.
		\item[b)] For arbitrary $\bset$, the two following assertions are mutually exclusive
		\begin{description}
			\item[i)] $\# \bset =1$ and $\kin$ is non-zero.
			\item[ii)] $\kds$ is not strictly positive definite on $\Sfin$. 
		\end{description}	
	\end{description} 
\end{prop}
\begin{prop}[(Strict) positive definiteness of $\kde$]
	\label{th_strict_k}
	Let us consider here again the notation of Proposition~\ref{prop1} and consider furthermore the class of kernels $\kde: (S,S') \in \Sfin \to 
	\kout \circ d_{\emb} (S,S')$ of Eq.~\ref{eq_proposed_k}, where $\kout:[0,\infty) \to \R$ is chosen such that $(h,h')\in \mathcal{H}^2 \to \kout(||h-h'||_{\mathcal{H}})$ is positive definite for any Hilbert space $(\mathcal{H}, \langle \cdot, \cdot, \rangle_{\mathcal{H}})$. Then, 
	\begin{description}
		\item[a)] $\kde$ is positive definite on $\Sfin$. 
		\item[b)] If furthermore $\kin$ satisfies $\textbf{i)}$ of condition \textbf{b)} in Proposition~\ref{prop1}, and $\kout: [0,\infty) \to \R$ is chosen such that $(h,h')\in \mathcal{H}^2 \to \kout(||h-h'||_{\mathcal{H}})$ is strictly positive definite for any Hilbert space $(\mathcal{H}, \langle \cdot, \cdot, \rangle_{\mathcal{H}})$, then $\kde$ is strictly positive definite on $\Sfin$.
	\end{description}
\end{prop}   

\begin{remark}
	As mentioned in \cite{bachoc2018gaussian3}, 
	continuous functions inducing strictly positive definite functions on any Hilbert space can be characterized following Schoenberg's works both in terms of completely monotone functions and of infinite mixtures of squared exponential kernels (See, e.g., \cite{Wendland2005}). 
\end{remark}

\subsection{Practicalities}
\label{prac}

In what follows and as in many practical situations, we consider ``inner'' (i.e., on $\bset$) kernels of the form $\kin(\xx,\xx')=\varin \rin(\xx,\xx')$, where $\varin >0$ and $\rin$ is a (strictly) positive definite kernel on $\bset$ taking the value $1$ on the diagonal and parametrized by some (vector-valued or other) hyperparameter $\psiin$. In such a case, denoting $\embr(S)=\frac{1}{\# S} \sum_{\xx\in S} \rin(\xx,\cdot)$ and $d_{\embr}$ the associated canonical distance, we immediately have that $\emb=\varin \embr$ and $d_{\emb}=\sigin d_{\embr}$. As a consequence, if $\kout(\cdot)$ writes $\varout \rout(\frac{\cdot}{\theta_{H}})$ for $\varout, \thetaout >0$ and $\rout(\cdot)$ defining a radial (strictly) positive definite kernel on any Hilbert space (possibly depending on some other hyperparameters ignored for simplicity) with $\rout(0)=1$, 
\begin{equation*}
\kde(S,S')=\varout \rout\left(\frac{\sigin}{\thetaout}d_{\embr}(S,S')\right),
\end{equation*} 
and it clearly appears that having both $\sigin$ and $\thetaout$ results in overparametrization of $\kde$. For this reason, we adopt the convention that $\sigin=1$, hence remaining with the hyperparameters $\varout$, $\thetaout$ and $\psiin$ to be fitted, possibly along with others such as trend and/or noise parameters. 
In our experiments, where noiseless settings and a constant trend are assumed, we appeal to Maximum Likelihood Estimation with concentration on the $\varout$ parameter and a genetic algorithm with derivatives \citep{mebane2011genetic}, in the flavour of the solution implemented in the DiceKriging R package \citep{roustant2012dicekriging}. 

In the numerical experiments presented next, the base set $\bset$ is assumed to be of the form $[0,1]^d$ (in our examples $d=2$), and we choose for $\rin$ an isotropic Gaussian correlation kernel solely parametrized by a ``range'' $\thetain$. 
As for $\rout$, while any kernel admissible in Hilbert space such as those of the Mat\'ern family would be suitable, we also choose here a Gaussian for simplicity, hence ending up with a triplet of covariance hyperparameters, namely $(\sigout,\thetaout, \thetain) \in (0,+\infty)^3$. 
As $\varout$ is taken care of by concentration (i.e. its optimal value for any given value of $\thetaout, \thetain$ can be analytically derived as a function of $\thetaout$ and $\thetain$), there remains to maximize the corresponding concentrated (a.k.a. profile) log-likelihood function with respect to $\thetaout$ and $\thetain$. 
For this purpose the analytical gradient of the concentrated log-likelihood with respect to these parameters has been calculated and implemented. Besides, parameter bounds need to be specified to the chosen optimization algorithm (i.e. \textit{genoud}, here) and while it seems natural to choose bounds in terms of $\sqrt{d}$, the diameter of the unit $d$-dimensional hypercube, for $\thetaout$ the adequate diameter is slightly less straightforward and calls for some analysis with respect to the range of variation of $d_{\embr}$ and how it depends on $\thetain$. The next proposition establishes simple yet practically quite useful results regarding the diameter of $\sset_r$ ($r>0$) with respect to $d_{\embr}$ and its maximal value when letting $\thetain$ vary.      
\begin{prop}
	\label{diameter_de}
	Let $\rin$ be an isotropic positive definite kernel on $\bset=[0,1]^d$ assumed to be monotonically decreasing to $0$ with respect to the Euclidean distance between elements of $\bset$, with range parameter $\thetain >0$. 
	Then the $d_{\embr}$-diameter of $\sset_p(\bset)$ $(p\geq 1)$, i.e. $\sup_{S,S' \in \sset_p}d_{\embr}(S,S')$, is reached with arguments $\{\mathbf{0}_d,\dots,\mathbf{0}_d\}$ and $\{\mathbf{1}_d,\dots,\mathbf{1}_d\}$, where $\mathbf{0}_d=(0,\dots,0), \mathbf{1}_d=(1,\dots,1) \in \bset$. Furthermore, the supremum of this diameter with respect to $\thetain \in (0,+\infty)$ is given by $\sqrt{2}$.  
\end{prop}

\section{Applications}

We now demonstrate the applicability of the class of DE kernels for both prediction and optimization purposes, with comparisons when applicable to similar methods based on DS kernels, 
and also to random search in the optimization case. In all examples, both inner and outer kernels  (resp. $\kin$ and $\kout$) are assumed Gaussian.
The three hyperparameters $(\sigout,\thetaout, \thetain)$ are estimated by Maximum Likelihood with concentration on $\sigout^2$, as detailed in Section~\ref{prac}. 
Three synthetic test functions and two application test cases are considered, respectively in mechanical engineering (CASTEM) and in hydrogeology (Contaminant source localization), all presented below. In the CASTEM case, the available data set consists of a fixed number ($404$) of (set input)-output instances, while in the other test cases one may boil down to a similar situation by restricting the scope to  finitely many such instances. 
Yet, the hydrogeology test case is the only one where $\bset$ is structurally restricted to remain finite, here a set of $25$ possible well locations, hence leading to a combinatorial optimization problem. 

\subsection{Presentation of Test Functions and Cases}

\subsubsection{Synthetic Functions}
Our three synthetic test functions consist of extensions of the Branin-Hoo test function \citep[See, e.g.,][]{roustant2012dicekriging}, denoted below by $g$, for set-valued inputs. These extensions are based respectively on the maximum (MAX), minimum (MIN), and mean (MEAN) of $g$ values associated with each of $p=10$ evaluation points in $\bset=[0,1]^2$, leading to    
\begin{equation}
f(S) = \max_{\xx\in S}g(\xx)
\label{eqn:maxfunc}
\end{equation}
\begin{equation}
f(S) = \min_{\xx\in S}g(\xx)
\label{eqn:minfunc}
\end{equation}
\begin{equation}
f(S) = \frac{1}{\#S}\sum_{\xx\in S}g(\xx),
\label{eqn:meanfunc}
\end{equation}
where $S \in \sset_p=([0,1]^2)^{10}$.  
Let us remark that by design, the $f$ of Eq.~\ref{eqn:meanfunc} is well-suited to be approximated using the double sum kernel of Eq.~\ref{def_dsk}.
Indeed, if $g$ is assumed to be a draw of a GP with kernel $\kin$, then $f$ is a draw of a GP with kernel $\frac{1}{\#S} \frac{1}{\#S'}\sum_{\xx\in S, \xx\in S'}\kin(\xx, \xx')$, as numerical results 
of Sections~\ref{predres} and~\ref{optres} do reflect.

\subsubsection{CASTEM Simulations}
The CASTEM dataset, inherited from \citep{ginsbourger2016design}, was originally generated from mechanical simulations performed using the Cast3m code \citep{Cast3Msoftware} to compute equivalent stress values on biphasic material subjected to uni-axial traction. 
The unit-square represents a matrix material containing 10 circular inclusions with identical radius of $R = 0.056419$.
The dataset consists of 404 point-sets along with their corresponding stress levels. Fig. \ref{fig:exampledata} illustrates two (set input)-output instances from it.
While the goal pursued in \citep{ginsbourger2016design} was rather in uncertainty propagation, we consider this data set here also from an optimization perspective.   
\begin{figure}[h!]
	\centerline{\includegraphics[scale=0.33]{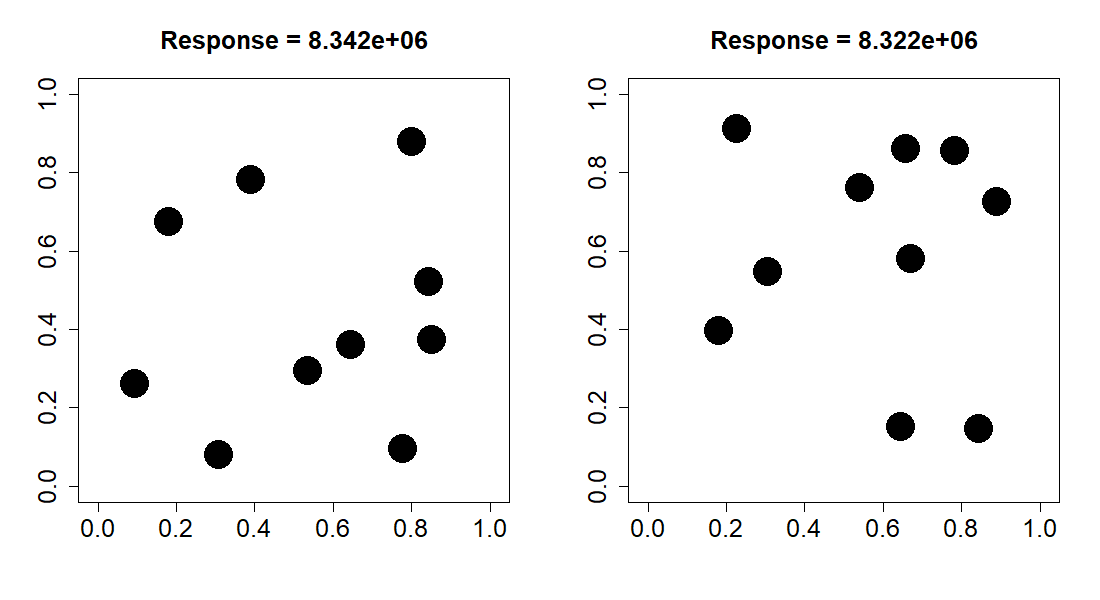}}
	\caption{Two CASTEM (set input)-output instances}
	\label{fig:exampledata}
\end{figure}

\subsubsection{Selection of Monitoring Wells for Contaminant Source Localization}

This test case relies on a benchmark generator of groundwater contaminant source localization problems from \citep{pirot2019contaminant}. 
The original problems consisted in finding among given candidate source localizations $\xx_i \in \R^2$ ($1\leq i \leq 2601$) which globally minimizes some measures of misfit between ``reference'' (or ``observed'') and ``simulated'' contaminant concentrations at fixed times and monitoring wells such as 
\begin{equation}
g(\xx,S)=\left(\sum_{i\in S}\sum_{t=1}^T|c_{\text{obs}}(i,t)-c_{\text{sim}}(\xx,i,t)|^2\right)^{\frac{1}{2}},
\end{equation} 
where $c_{\text{obs}}(i,t)$ is the reference concentration at well $i$ and time step $t$, $c_{\text{sim}}(\bold{x},i,t)$ is the corresponding simulated concentration when the source of contaminant is at $\xx$, and $S\subset S_{\text{full}}:=\bset=\{1,2, \dots, 25 \}$ is a given subset from $25$ fixed monitoring wells.  

Here, instead of fixing the subset of well locations $S$ and looking for the optimal $\xx$, we consider instead the maps of score discrepancies $g(\cdot, S_{\text{full}})-g(\cdot, S)$ as a function of $S$. From there, the considered combinatorial optimization problem consists in minimizing 
\begin{equation}
\label{eq:cont_obj}
f(S)=\sum_{i=1}^{2601}(g(\xx_i,S_{\text{full}})-g(\xx_i,S))^2
\end{equation}
over 
the set $\sset_p(\bset)$ of subsets of $p<25$ wells from $\bset$.
In the numerical experiments,  we fix $p=5$, and hence the cardinality of the considered set of subsets $\sset_5(\bset)$ is $\binom{25}{5} = 53,130$. To test the efficiency of our approach on this application, the two contaminant source locations (A and B) and two geological geometries of \citep{pirot2019contaminant} are considered, leading to four cases (denoted (Src A, Geo 1), (Src A, Geo 2), (Src B, Geo 1), (Src B, Geo 2), respectively). 

Since the base set $\bset=\{1,2, \dots, 25 \}$ is itself finite here, it follows from Prop.~\ref{th_nonstrict_dsk} that resulting double sum kernels are not strictly positive definite so that BO with those kernels fails after few iterations, as found in numerical experiments. 
Two subsets of five well locations are plotted in Fig.~\ref{fig:Contexampledata} along with contours of corresponding score discrepancy maps $g(\cdot, S_{\text{full}})-g(\cdot, S)$ and values of objective function $f$ derived from them. 

The first combination (left subfigure) better represents the misfit function $g(\cdot, S_{\text{full}})$ overall with a lower $f$ value. This 
subset is in fact the optimal one, obtained by exhaustive search over all $53,130$ candidates.
Our goal is precisely to efficiently locate by BO these optimal well locations whose contributions minimize the spatial sum of score discrepancies. 
%
The reader is referred to \citep{pirot2019contaminant} for further details and visualization of the misfit function, location of the contaminant source, and coordinates of well locations.

\begin{figure*}[t]
	\centering
	\begin{minipage}{0.5\textwidth}
		\includegraphics[width=1\textwidth]{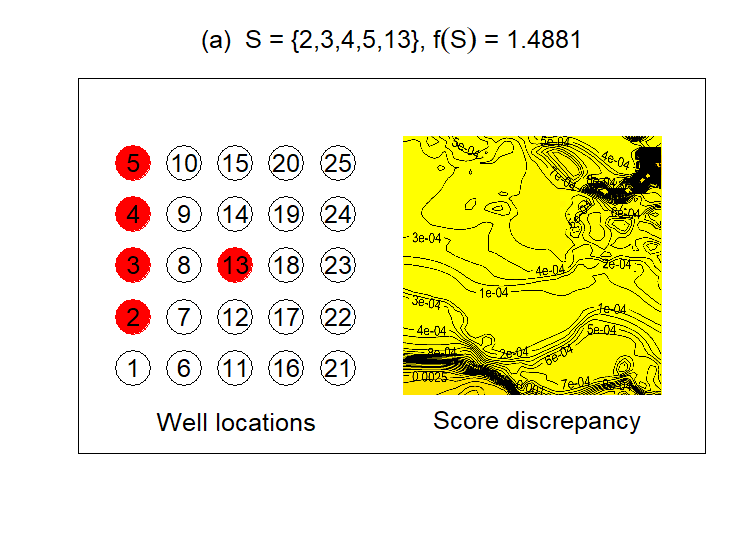}
	\end{minipage}%
	\begin{minipage}{0.5\textwidth}
		\includegraphics[width=1\textwidth]{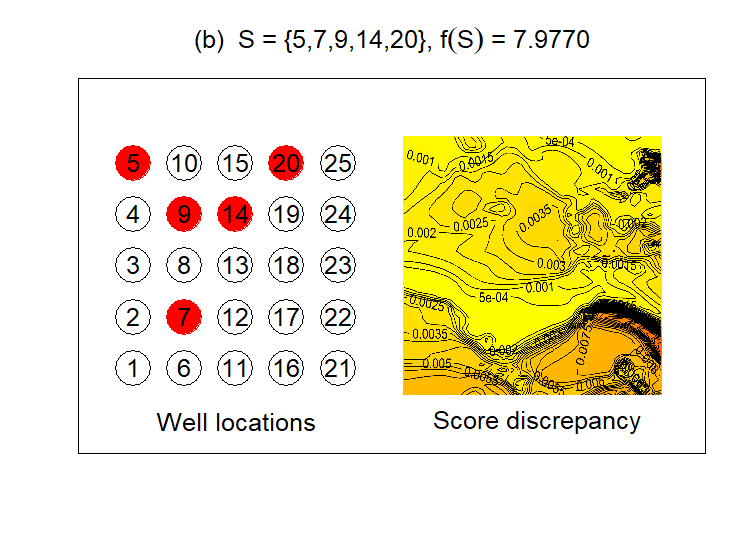}
	\end{minipage}%
	\caption{Score discrepancy map: location of selected wells (input $S$), score discrepancy landscape, and the spatial sum of score discrepancy objective function value $f(S)$.}
	\label{fig:Contexampledata}
\end{figure*}

\subsection{Prediction: Settings and Results}
\label{predres}

To assess the predictive ability of the considered GP models under the considered settings of data sets split into learning and test parts, we appeal to the so-called $Q^2$ or ``predictive coefficient'' \citep{Marrel.etal2008}, 
\begin{equation}
Q^2 = 1-\frac{\sum_{i=1}^{n_{\text{test}}}(f(S_i^{\text{(test)}})-m_n(S_i^{\text{(test)}}))^2}{
	\sum_{i=1}^{n_{\text{test}}}(f(S_i^{\text{(test)}})-\bar{\textbf{f}})^2},
\label{eqn:Q2}
\end{equation}
where $n_{\text{test}}$ is the number of test point-sets, $f(S_i^{\text{(test)}})$ and $m_n(S_i^{\text{(test)}})$  are the actual response and the mean values predicted by the GP model, respectively. $\bar{\textbf{f}}$ is the mean of $f(S_i^{\text{(test)}})$'s. The closer to $1$ the value of $Q^2$, the more efficient the predictor is. In addition, we also look at visual diagnostics based on the comparison of standardized residuals (i.e. divided by GP prediction standard deviations) with the normal distribution, both in cross- and external validation. 

As a result of Prop. \ref{th_nonstrict_dsk}, the DS kernel is not readily applicable for the contaminant source localization test case, due to singularity issues with covariance matrices. One way around this is to add a small positive jitter to their diagonal \citep[inspired by][]{ranjan2011computationally}. This approach will be referred to hereafter as DS+j whenever it is used in place of the original DS. More detail on the procedure used for jitter tuning and additional results can be found in supplementary material.

The total size of datasets used to assess prediction performances for the three synthetic test problems, CASTEM, and the contamination applications are 1000, 404, and 200, respectively. Each dataset is further partitioned into training and testing sub-datasets with percentages (80:20), (50:50) and (20:80).  
Average $Q^2$ values over 20 replications are provided in Table \ref{Q2-table1}. First, we observe that $Q^2$ tends to increase with the proportion of the full data set used for training, except in one case with CASTEM.  We see that the proposed approach with the DE kernel gives higher value of $Q^2$ than that with the DS kernel on all problems except for the MEAN function. We hypothesize the latter to be due to the adequacy between the MEAN function's nature and the DS kernel, as remarked earlier.

Finally, Fig.~\ref{fig:cont2080} shows leave-one-out (left panel) and out-of-sample diagnostics (right panel) for the source localization application (Src A, Geo 1) with DE kernel. The results show relatively moderate departures from the normality assumptions. Complete residual analysis for all scenarios as well as for DS kernels (with jitter) can be found in supplementary material. 

\begin{table*}[t!]
	\centering
	\footnotesize
	\caption{$Q^2$ values for GP predictions on all test cases with DE versus DS kernels ($\kde$ versus $\kds$(+j))} \label{Q2-table1}
	\begin{tabular}{|l|c|c|c|c|c|c|}
		\hline
		\multirow{2}{*}{\textbf{Problem}} & \multicolumn{3}{c|}{$\kde$
		} & \multicolumn{3}{c|}{$\kds$}
		\\ \cline{2-7} 
		& \textbf{20:80} & \textbf{50:50} & \textbf{80:20} & \textbf{20:80} & \textbf{50:50} & \textbf{80:20} \\ \hline
		\textbf{(a) MAX} & 0.6926 & 0.8011 & 0.8559 & 0.5644 & 0.7429 & 0.7725 \\ \hline
		\textbf{(b) MEAN} & 0.9996 & 0.9999 & $\sim$1 & $\sim$1 & $\sim$1 & $\sim$1 \\ \hline
		\textbf{(c) MIN} & 0.3309 & 0.4582 & 0.4929 & 0.1080 & 0.2245 & 0.2749 \\ \hline
		\textbf{(d) CASTEM} & 0.5799 & 0.6641 & 0.6543 & 0.5067 & 0.5410 & 0.5056 \\ \hline
		\multirow{2}{*}{\textbf{Problem}} & \multicolumn{3}{c|}{$\kde$}
		& \multicolumn{3}{c|}{$\kds$+j}
		\\ \cline{2-7} 
		& \textbf{20:80} & \textbf{50:50} & \textbf{80:20} & \textbf{20:80} & \textbf{50:50} & \textbf{80:20} \\ \hline
		\textbf{(e) (Src A, Geo 1)} & 0.7607 & 0.9133 & 0.9352 & 0.7437 & 0.8445 & 0.8804 \\ \hline
		\textbf{(f) (Src A, Geo 2)} & 0.7239 & 0.8855 & 0.9240 & 0.7130 & 0.8485 & 0.8729 \\ \hline
		\textbf{(g) (Src B, Geo 1)} & 0.7977 & 0.9190 & 0.9447 & 0.7901 & 0.8746 & 0.8904 \\ \hline
		\textbf{(h) ()Src B, Geo 2)} & 0.8486 & 0.9151 & 0.9439 & 0.8389 & 0.8944 & 0.9252 \\ \hline
	\end{tabular}
\end{table*}

\begin{figure*}[t]
	\centering
	\begin{minipage}{0.5\textwidth}
		\includegraphics[width=1\textwidth]{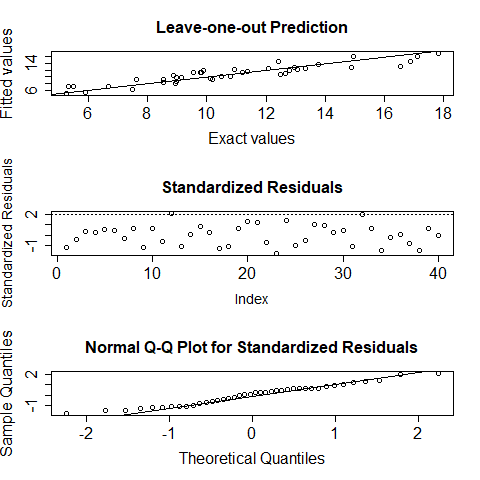}
	\end{minipage}
	\begin{minipage}{0.5\textwidth}
		\includegraphics[width=1\textwidth]{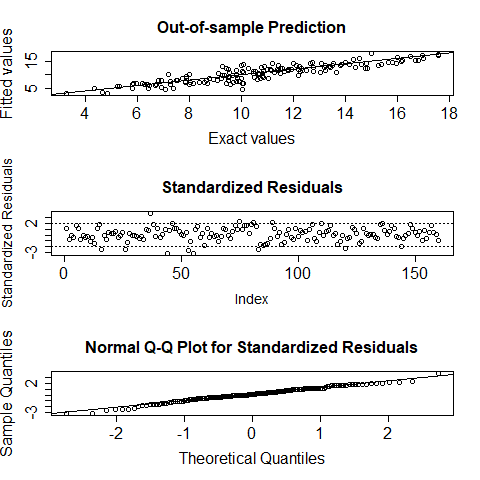}
	\end{minipage}
	\caption{GP prediction residual analysis on the contaminant source localization problem (Src A, Geo 1) with kernel $\kde$ and ratio (20:80). (a) Internal errors (left); (b) External errors (right).} \label{fig:cont2080}
\end{figure*}

\subsection{Optimization: Settings and Results}
\label{optres}

In this section, the efficiency of DE versus DS kernels (possibly with jitter) are evaluated within the BO framework, using the Expected Improvement (EI) \citep{mockus1978application} as infill sampling criterion. 
To assess optimization performances, the same datasets as those used in previous section are used for the three synthetic problems and CASTEM. 
As for the contaminant source application, the whole dataset of size $53,130$ is employed. Optimization performances are assessed on 50 repetitions of EI algorithms with $10$ initial design point-sets. For each repetition, all algorithms start with the same initial design, and are allocated $40$ additional objective function evaluations. The hyperparameters are iteratively re-determined in every iteration using MLE (See Section~\ref{prac} and supplementary material). 

Concerning EI maximization, EI values are computed at all point-sets and the one attaining the highest value is selected (no ties occurred). The performance is measured by (1) counting the number of trials (out of 50) for which the algorithm could find the best point from the considered dataset; and (2) monitoring the distribution of best found responses over iterations.  A random sampling method is used as baseline. Table \ref{BO-table} summarizes the number of trials that the minimum is found and Fig. \ref{BO-results} represents progress curves in terms of median and 95th percentile values of current best objective function values over 50 trials. 

\begin{table*}[t!]
	\footnotesize
	\caption{Numbers of trials (out of 50) for which the minimum is found for EI algorithms based on GP models with DE versus DS kernels, as well as for Random Sampling.} \label{BO-table}
	\centering
	\begin{tabular}{|l|c|c|c|}
		\hline
		\textbf{Problem}
		& \textbf{EI-$\kde$} & \textbf{EI-$\kds$} & \textbf{RANDOM} \\ \hline
		\textbf{(a) MAX} & 36 & 8 & 6 \\ \hline
		\textbf{(b) MEAN} & 50 & 50 & 4 \\ \hline
		\textbf{(c) MIN} & 9 & 8 & 3 \\ \hline
		\textbf{(d) CASTEM} & 28 & 10 & 5 \\ \hline
		\textbf{Problem}
		& \textbf{EI-$\kde$} & \textbf{EI-$\kds$+j} & \textbf{RANDOM} \\ \hline
		\textbf{(e) (Src A, Geo 1)} & 50 & 48 & 0 \\ \hline
		\textbf{(f) (Src A, Geo 2)} & 34 & 25 & 0 \\ \hline
		\textbf{(g) (Src B, Geo 1)} & 50 & 47 & 0 \\ \hline
		\textbf{(h) (Src B, Geo 2)} & 43 & 44 & 0 \\ \hline
	\end{tabular}
\end{table*}

EI algorithms with any of the two considered kernel classes clearly appear here superior to random sampling. Experiments on synthetic problems show that within the two considered EI algorithm settings, DE kernels outperform DS ones on the MAX problem both in terms of the number of trials that the true minimum is found and of the final best responses. On the MEAN problem, though, while both approaches lead to locate the minimum for all 50 replications, DS kernels lead to a fewer number of iterations as anticipated due to adequacy between this kernel class and the test function.  EI algorithms with both kernels did not perform well on the MIN problem which may be explained by the fact that the underlying Branin-Hoo function has the large portion of the search space being quite flat. For the CASTEM dataset, EI-$\kde$ and  EI-$\kds$ methods could locate the minimum for 28 and 10 trials, respectively, against $5$ for random sampling. 

\begin{figure*}[h!]
	\centering
	
	\begin{minipage}{0.25\textwidth}
		\includegraphics[width=1\textwidth]{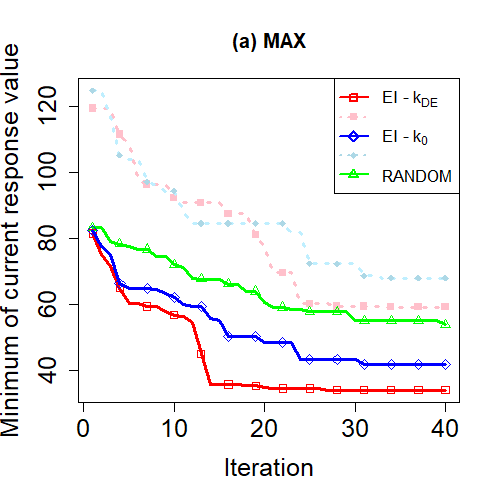}
	\end{minipage}%
	\begin{minipage}{0.25\textwidth}
		\includegraphics[width=1\textwidth]{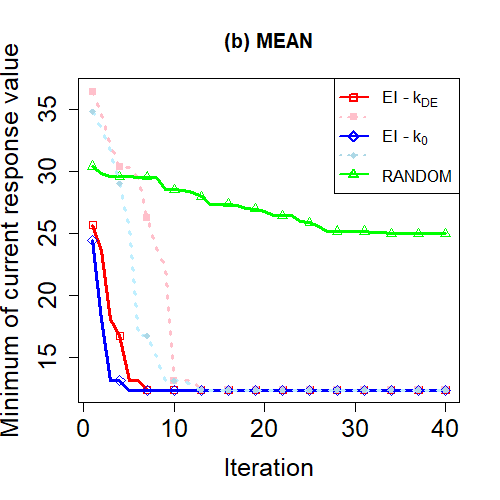}
	\end{minipage}%
	\begin{minipage}{0.25\textwidth}
		\includegraphics[width=1\textwidth]{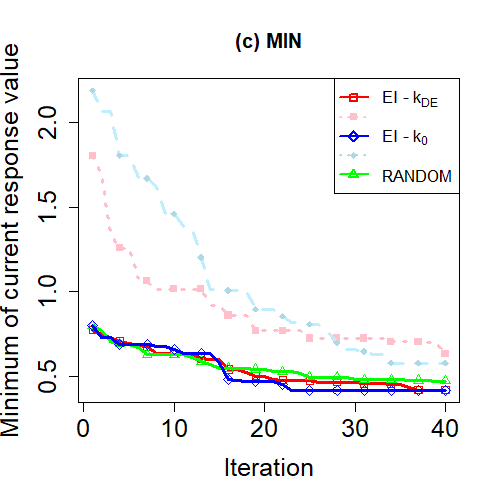}
	\end{minipage}%
	\begin{minipage}{0.25\textwidth}
		\includegraphics[width=1\textwidth]{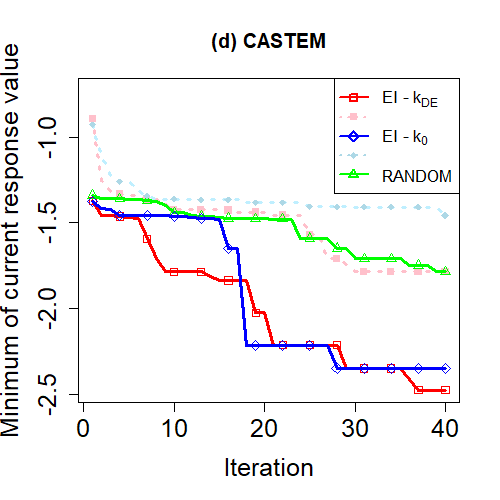}
	\end{minipage}%

	\begin{minipage}{0.25\textwidth}
		\includegraphics[width=1\textwidth]{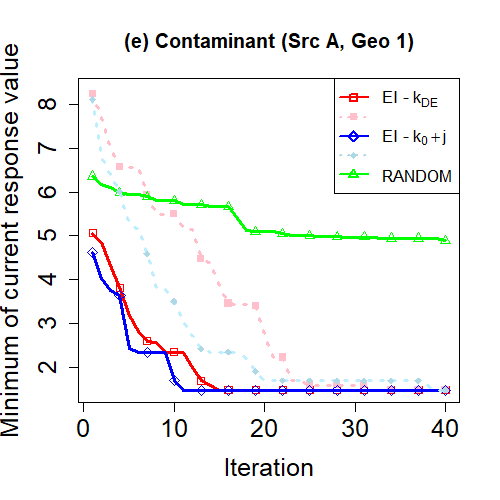}
	\end{minipage}%
	\begin{minipage}{0.25\textwidth}
		\includegraphics[width=1\textwidth]{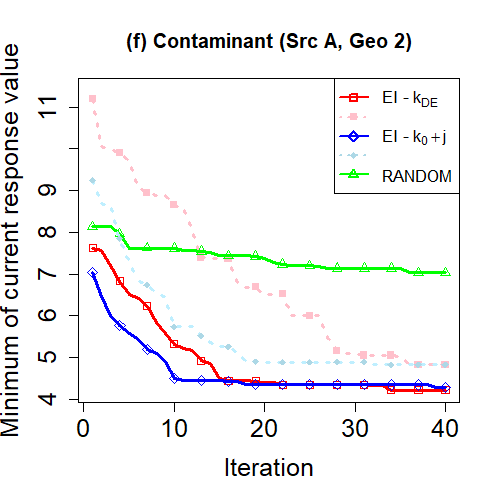}
	\end{minipage}%
	\begin{minipage}{0.25\textwidth}
		\includegraphics[width=1\textwidth]{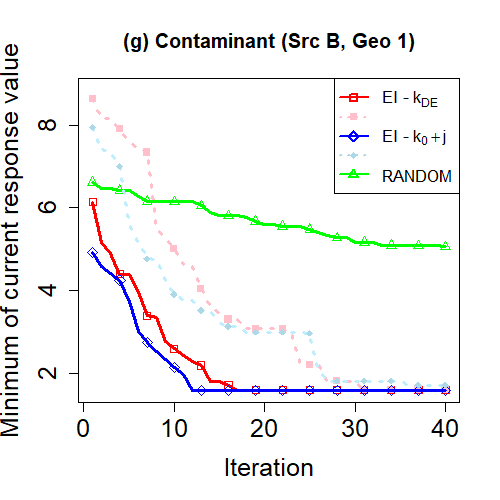}
	\end{minipage}%
	\begin{minipage}{0.25\textwidth}
		\includegraphics[width=1\textwidth]{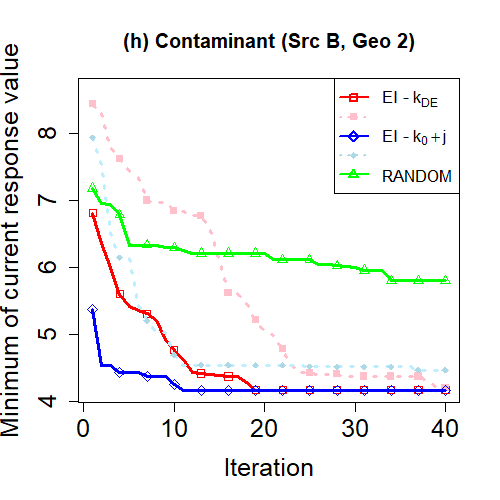}
	\end{minipage}%
	\caption{BO progress curves showing the median (solid lines) and 95th percentile (dotted lines) values of the current best response of problems (a) MAX, (b) MEAN, (c) MIN, (d) CASTEM, and contaminant problems (e) (Src A, Geo 1), (f) (Src A, Geo 2), (g) (Src B, Geo 1) and (h) (Src B, Geo 2).}
	\label{BO-results}
\end{figure*}

As for the source localization application, the obtained EI-$\kds$ results are all involving the use of a jitter. Overall, EI algorithms coupled with either of the two kernel classes appeared by far better than random sampling. Comparing performances between the two EI algorithms, EI-$\kde$ method could locate the global optimum more frequently (as indicated in Table \ref{BO-table}). In particular, with the DE kernel, the EI algorithm found the global optimum in every trial run on two scenarios of contaminant source localization problems (i.e. (Src A, Geo 1) and (Src B, Geo 1)).

The median progress curves (bottom panel of Fig. \ref{BO-results}) illustrate on the other hand that the DS kernel seem quite well-suited for the contaminant problems, as highlighted in particular by the fast initial decrease in best response value.  
The 95\% quantile curves suggest however that in the worst situations, EI-$\kde$ performs relatively better and seems to be more robust especially toward the end of the course when the jitter was needed to make EI-$\kds$ work. 
It is worth noting that determining an appropriate jitter level to add to the DS kernel is not a straightforward task. While one would want to add a smallest possible value of jitter, oftentimes, a too small jitter is not enough to fix conditioning issues. Additional results, with a large number of trials, revealing the effect of a poor choice of jitter level on DS kernel model's accuracy as well as optimization results are given in supplementary material. 
Overall, the strict positive definiteness of considered DE kernels (and the fact that no jitter is required) make them appear as a relatively robust option to efficiently address expensive combinatorial optimization problems in a ``black-box'' Bayesian Optimization framework (i.e., without requiring much prior knowledge about the problem structure).

\section{Discussion}

Experimental results obtained on the analytical objective functions and application test cases confirm the added value of the considered approaches for set-function prediction and (combinatorial) optimization. 

Yet a number of challenges and potential extensions remain to be addressed in future work. This includes computational difficulties that will arise when working with larger numbers of subsets and/or subset cardinalities, not only to handle bigger matrices but also to tackle the optimization of infill criteria. These criteria include the Expected Improvement as well as adaptations of further families of BO acquisition functions from frameworks such as Predictive Entropy Search \citep{Hern'andez-Lobato.etal2014}, Knowledge Gradient \citep{Frazier2018}, and others. 

From the test case perspective, future work may also include tackling further prediction and subset selection problems (be it in continuous or combinatorial settings, with problem structures of various levels of complexity), not only for optimization purposes but also with more general goals around uncertainty quantification and reduction \citep{Bect.etal}. Besides this, a nice feature of the considered approaches is that they would naturally extend to cases with varying subset cardinalities and also with ``marked'' point sets (in the vein of \citep{Cuturi.etal2005}'s molecular measures), hence accommodating applications such as CASTEM but with varying inclusion numbers and radii.  
Furthermore, the conceptual approach of chaining an embedding and a kernel in Hilbert space (also in the flavour of \citep{Christmann2010}) could apply to a variety of other input types provided that relevant mappings to Hilbert space can be found, opening the door to numerous non-conventional extensions of GP-based prediction, BO, and related kernel methods.

\subsubsection*{Acknowledgements}
The authors would like to thank the anonymous referees for constructive comments having lead to substantial improvements of the paper. 
P.B. would like to thank DPST scholarship project granted by IPST, Ministry of Education, Thailand for providing financial support during his master study. 
D.G.'s contributions have taken place within the Swiss National Science Foundation project number 178858. 
Furthermore, D.G. would like to thank colleagues including notably Fabrice Gamboa, Ath\'ena\"is Gautier, Luc Pronzato, Henry Wynn, and Anatoly Zhigljavsky for enriching discussions in recent years around ideas presented in this paper.
T.K. would like to acknowledge the support of Thailand Research Fund under Grant No.: MRG6080208, Centre of Excellence in Mathematics, CHE, Thailand, and the Faculty of Science, Mahidol University.
The authors would like to acknowledge the support of Idiap Research Institute. In particular, most numerical experiments presented here were run on Idiap's grid.   
The authors also thank Drs. Jean Baccou and Fr\'ed\'eric Perales 
(Institut de Radioprotection et de S\^uret\'e Nucl\'eaire,
Saint-Paul-l\`es-Durance, France) for the CASTEM data, and 
Dr. Cl\'ement Chevalier who has been involved in investigations on this data in the
framework of the ReDICE consortium.

\bibliographystyle{apalike}

\newpage
\renewcommand\appendixname{Supplementary Material}
\appendix 

\section{Elements of literature review}

Before reviewing some foundational machine learning papers dealing with kernels on sets of (sub)sets and related objects, let us start by some preliminary remarks on how an elementary class of positive definite kernels can be constructed in the context of measure spaces and why these kernels are not necessarily ideal for the prediction and optimization objectives we have in mind. Consider here a set $\bset$ equipped with a sigma-algebra $\mathcal{A}$ and a measure $\mu$, making it a measure space $(\bset, \mathcal{A}, \mu)$. Then it comes without much effort that the mapping $k$ defined by 
\begin{equation*}
k:(S,S') \in \mathcal{A}^2 \to \mu(S\cap S') \in [0,\infty)
\end{equation*}
constitutes a positive definite kernel. Indeed, taking arbitrary $n\geq 1$, $a_1,\dots,a_n \in \R$, $S_1,\dots,S_n \in \mathcal{A}$ and recalling that $\mu(S\cap S')=\int_{\bset} \mathbf{1}_{S}(\mathbf{u}) \mathbf{1}_{S'}(\mathbf{u}) \mathrm{d}\mu (\mathbf{u})$ , we do have 
\begin{equation*}
\begin{split}
\sum_{i=1}^n \sum_{j=1}^n a_i a_j k(S_i,S_j) 
&= \int_{\bset} \left( \sum_{i=1}^n a_i \mathbf{1}_{S_{i}}(\mathbf{u}) \right)^2 \mathrm{d}\mu (\mathbf{u})  \geq 0
\end{split}
\end{equation*}

In the particular case where $\bset$ is finite, $\mathcal{A}$ is the associated power set $\mathcal{P}(\bset)$, and $\mu$ is the counting measure, we find that 
\begin{equation*}
k(S,S')=\#(S\cap S')=\sum_{\xx\in S} \sum_{\xx'\in S} \frac{1}{2}\delta_{\xx,\xx'},
\end{equation*}
a kernel that does account for the position of points only to the extent that it counts the number of points simultaneously in both sets (without any account for the closeness of non-coinciding points). 
Such a kernel is referred to as \textit{default kernel on sets} in \citep[][Example 4.2]{Gaertner2004}, where it appears as a particular case of an abstract construction denoted \textit{default kernel for basic terms} (Definition 4.1, 
p. 213) and that is also applied for instance to multisets (Example 4.3 of the same page). For the case of the default kernel on sets, the authors comment following Example 4.2 that \textit{``the intuition here is that using the matching kernel for the elements of the set corresponds to computing the cardinality of the intersection of the two sets. Alternatively, this computation can be seen as the inner product of the bit-vectors representing the two sets''}. 

\bigskip 

Yet another important class of kernels for structured data, notably put to the fore by \cite{Gaertner2004} yet by pointing out high associated computational costs, is 
the class of \textit{convolution kernels} dating back to \cite{Haussler1999}. Convolution kernels can accommodate a variety of so-called ``composite structures'' by relying on their respective ``parts''. They are constructed based on prescribed kernels between vectors of parts by instantiating and summing them with respect to all vectors of parts generating the considered compositive structures (Theorem~1 in \cite{Haussler1999}). The proof of the latter theorem turns out to be based on the following Lemma that focuses on composite structures writing as finite subsets of a base set (say $\bset$, to stick to the notation of the present paper): 
\begin{prop}[Lemma~1 of \cite{Haussler1999}]
	Let $k$ be a kernel on $\bset \times \bset$ and for all finite, nonempty $A,B	\subseteq \bset$ define $k'(A,B)=\sum_{x\in A, y\in B}k(x,y)$. Then $k'$ is a kernel on the product of the set of all finite, nonempty, subsets of $\bset$ with itself. 	
\end{prop}

Let us remark that this construction is none other than what we refer to as the double sum kernels throughout the paper, notably at the heart of \citep{Kim}. 

In contrast, the approach employed in \citep{Kondor2003} to create classes of kernels between sets consists in viewing these sets as samples from multivariate Gaussian distributions and then defining their baseline kernel in terms of  Bhattacharyya affinity between those distributions. The resulting approach is then further enriched or ``kernelized'' thanks to the introduction of a second kernel defined between elementary vectors. 
In \cite{Cuturi.etal2005}, the focus is on kernels on measures characterized by the fact that the value of the kernel between two measures is a function of their sum, and the proposed constructions rely on common quantities defined on measures such as entropy or generalized variance.
Quoting the article, \textit{``the considered kernels can be used to derive kernels on structured objects, such as images and texts, by representing these objects as sets of components, such as pixels or words, or more generally as measures on the space of components''}.
Here again, given an other kernel on the space of components itself, the approach is further extended using the ``kernel trick''. 

\bigskip 

\cite{Christmann2010} investigate universal kernels on non-standard input spaces.
They consider in particular a kernel on the set of probability measures obtained by chaining a radial Gaussian kernel and the RKHS distance between embedded distributions, coinciding in the case of uniform distributions over finite sets with our proposed class of \textit{Deep Embedding} kernels. They show that in case of a compact base space and with probability measures endowed with the topology of weak convergence, the kernels of interest are universal. 
The reader is also referred to  \citep{Berlinet.Thomas-Agnan2004, Smola2007, Sriperumbudur2011, Muandet2017} and references therein for more background results on RKHS embeddings of probability measures. Besides this, RKHS embeddings are also at the heart of the thesis \cite{Sutherland2016}, focusing on ``Scalable, Flexible and Active Learning on Distributions''. 
Kernel distribution embeddings have been recently further studied in \cite{Simon-Gabriel2018} from a functional analysis perspective, resulting in a proof that for kernels, being \textit{universal}, \textit{characteristic}, and \textit{strictly positive definite} (where the definitions are slightly extended) are essentially equivalent.  
The latter paper gives furthermore a complete characterization of kernels whose associated Maximum Mean Discrepancy distance metrizes weak convergence, and it is shown in turn that kernel mean embeddings can be extended from probability measures to Schwartz distributions. 

\section{Proofs of theoretical results}

\begin{proof}[Proof of Prop.~\ref{prop1}]
	\textbf{a)} $k(S,S')=\langle \emb(S), \emb(S') \rangle_{\Hkin}$ $(S,S'\in \Sfin)$ follows directly from scalar product bilinearity and $\langle \kin(\xx,\cdot), \kin(\xx',\cdot) \rangle_{\Hkin}=\kin(\xx,\xx')$  $(\xx,\xx'\in \bset)$, by reproducing property. Positive definiteness is then inherited from the scalar product as, for any $n\geq 1$, $a_1,\dots, a_n \in \R$ and $S_1,\dots, S_n \in \Sfin$, $\sum_{i=1}^n\sum_{j=1}^n a_i a_j k(S_i,S_j)=\left|\left| \sum_{i=1}^n a_i \emb(S_i) \right| \right|_{\Hkin}^2 \geq 0$. 
	Similarly, the non-negativity, symmetry, and triangle inequality for $d_{\emb}$ are inherited from the metric $\left|\left| \cdot \right| \right|_{\Hkin}$, making the former a pseudometric on $\Sfin$.
	\textbf{b)} First, \textbf{ii)} $\Leftrightarrow$ \textbf{iii)} as $d_{\emb}(S,S')=||\emb(S)-\emb(S')||_{\Hkin}$ and $\textbf{ii)}$ means that for $S\neq S'$ $\emb(S)\neq\emb(S')$, or equivalently $||\emb(S)-\emb(S')||_{\Hkin}\neq 0$ for $S\neq S'$, which is exactly what is needed for the pseudo-metric $d_{\emb}$ to qualify as a metric on $\Sfin$. 
	\textbf{i)} $\Rightarrow$ \textbf{ii)}: Let $S=\{\yy_1,\dots,\yy_{n_1}\}$ and $S'=\{\yy_1,\dots,\yy_{n_2}\}$ be distinct elements of $\Sfin$. Let us denote by $\ell \geq 0$ ($\ell \leq n_1+n_2$) the number of elements in $S\cap S'$ and denote $n=n_1+n_2-\ell$ and by $\xx_1, \dots, \xx_n$ the elements of $S\cup S'$ ordered so as to have as first $n_1-\ell$ elements those of $S\backslash S'$, then the $\ell$ elements from $S\cap S'$, and finally those of $S'\backslash S$ (the orders within those three categories being arbitrary). Denote further here $\mathbf{X}_n=(\xx_1, \dots, \xx_n)$.
	Then, 
	\begin{equation*}
	\label{eq_inj}
	\begin{split}
	&\emb(S)- \emb(S') 
	=\frac{1}{n_1} \sum_{i=1}^{n_1-\ell} \kin(\xx_i,\cdot)\\
	&+\left( \frac{1}{n_1} -\frac{1}{n_2}\right) \sum_{i=n_1-\ell+1}^{n_1} \kin(\xx_i,\cdot)
	+ \frac{1}{n_2} \sum_{i=n_1+1}^{n} \kin(\xx_i,\cdot),
	\end{split}
	\end{equation*}
	whereof, putting $a_i=\frac{1}{n_1} \ (1\leq i \leq n_1-\ell)$,
	$a_i=\frac{1}{n_1} -\frac{1}{n_2} \ (n_1-\ell+1\leq i \leq n_1)$,
	$a_i=\frac{1}{n_2} \ (n_1+1\leq i \leq n)$, and noting 
	$\kin(\mathbf{X}_n)=(\kin(\xx_i,\xx_j))_{i,j \in \{1,\dots,n\}}$, we have
	\begin{equation*}
	\label{eq_inj2}
	\begin{split}
	\left|\left| \emb(S)- \emb(S') \right| \right|_{H_{\kin}}=\sqrt{\mathbf{a}'\kin(\mathbf{X}_n)
		\mathbf{a}} >0
	\end{split}
	\end{equation*}
	where $\mathbf{a}=(a_1,\dots,a_n) \in A_n$ and the positivity follows from $\textbf{i)}$, implying that $\emb(S)\neq \emb(S')$ indeed. Assuming now that $\textbf{ii)}$ holds and considering elements $\xx_1, \dots, \xx_n \in \bset$ and $\mathbf{a}=(a_1,\dots,a_n) \in A_n$ such as in $\textbf{i)}$ (with $\ell, n_1,n_2$ following from $\mathbf{a}$), we define this time $S=\{\xx_1,\dots,\xx_{n_1+\ell}\}$ and $S'=\{\xx_{n_1+1},\dots,\xx_{n}\}$ and conclude that $\textbf{i)}$ holds by pointing out that $\sum_{i=1}^n \sum_{j=1}^n a_i a_j \kin(\xx_i,\xx_j) =\left|\left| \emb(S)- \emb(S') \right| \right|_{H_{\kin}} >0$, where $S\neq S'$ follows from the assumption of pairwise distinct $\xx_i$'s. 
\end{proof}

\begin{proof}[Proof of Prop.~\ref{th_nonstrict_dsk}]
	\textbf{a)} Putting $\kin(\mathbf{X}_c)=(\kin(\xx_i,\xx_j))_{i,j \in \{1,\dots,c\}}$ and \\ $u(S)=\frac{1}{\# S} (\mathbf{1}_{\xx_i \in S})_{1\leq i \leq c}$ in the right hand side directly delivers that 	
	\begin{equation*}
	u(S)^{T} \kin(\mathbf{X}_c) u(S)
	=\sum_{i=1}^c \sum_{j=1}^c \mathbf{1}_{\xx_i \in S}  \mathbf{1}_{\xx_j \in S'} \frac{\kin(\xx_i,\xx_j)}{\# S \# S'},
	\end{equation*}	
	which coincides indeed with Eq.~\ref{def_dsk}'s $\kds(S,S')$. Eq.~\ref{compact_dsk_matrix} then simply follows as a Gram matrix associated with the bilinear form defined by Eq.~\ref{compact_dsk}. 	
	\textbf{b)} That \textbf{i)} $\Rightarrow \overline{\mbox{\textbf{ii)}}}$ follows from the fact that if $\bset=\{\xx\}$ has cardinality $1$ and $\kin$ is strictly positive definite on $\bset$, then $\Sfin$ consists of the single element $\{\xx\}$, and $k(\{\xx\},\{\xx\})=\kin(\xx,\xx)>0$ whereof $k$ is strictly positive definite on $\Sfin$.  
	To prove that \textbf{ii)} $\Rightarrow \overline{\mbox{\textbf{i)}}}$, let us now consider the case where $\bset$'s cardinality is at least $2$ (finite or not). From this assumption, it is possible to choose two distinct elements in $\xx_A,\xx_B \in \bset$; let us denote here  $\mathbf{X}=\{\xx_A, \xx_B\}$, and set $S_1=\{\xx_A\}$, $S_2=\{\xx_B\}$, $S_{3}=\{\xx_A,\xx_B\}$, and $\mathbf{S}=(S_1,S_2,S_3)$. Following the same route as for  Eq.~\ref{compact_dsk_matrix}, we then get  
	\begin{equation*}
	\kds(\sets)=U(\sets)^{T}\kin(\mathbf{X}) U(\sets)
	=M(\sets)^{T} M(\sets),
	\end{equation*}
	with $M(\sets)=\kin(\mathbf{X})^{\frac{1}{2}}U(\sets)$. \\
	Hence $\rank(\kds(\sets)) \leq \rank(\kin(\mathbf{X})^{\frac{1}{2}})=\rank(\kin(\mathbf{X})) \leq 2$ and so the $3\times 3$ matrix $\rank(k(\sets))$ is non-invertible, proving indeed that $k$ is not strictly positive definite on $\Sfin$.   
\end{proof} 

\begin{remark}
	The first equation of point \textbf{a)} highlights the fact that even if $\kin(\mathbf{X})$ is a positive definite matrix (in particular, assuming that $\kin$ is strictly p.d. on $\bset$), the matrix $\kds(\sets)$ will actually be systematically singular for $q > c$. It turns out to also possibly happen in situations where $q \leq c$, as is for instance the case with $c=5, q=4$, and 
	$U(\sets) \propto 
	\left(
	\begin{array}{ccccc}
	1&1&0&0&1 \\
	0&0&1&1&1 \\
	1&0&0&1&1 \\
	0&1&1&0&1
	\end{array}
	\right)
	$.   
\end{remark}

\smallskip 

\begin{proof}[Proof of Prop.~\ref{th_strict_k}]
	Both points essentially rely on the fact that $d_\emb(S,S')=||\emb(S)-\emb(S')||_{\Hkin}$ and that, as Reproducing Kernel Hilbert Space, $\Hkin$ is in the first place a Hilbert space. Indeed, writing $\kde(S,S')=\kout(||\emb(S)-\emb(S')||_{\Hkin})$, we then directly obtain \textbf{a)} by composition of the positive definite kernel $(h,h')\in \mathcal{H}^2 \to \kout(||h-h'||_{\Hkin})$ with the mapping $\emb: \Sfin \mapsto \Hkin$. As for \textbf{b)}, assuming furthermore $\kout$ to be strictly positive definite on any Hilbert space and $\textbf{i)}$ of condition \textbf{b)} in Proposition~\ref{prop1} to hold, then the   
	strict positive definiteness of $\kde$ follows from the one of $\kout$ and the injectivity of $\emb$ ensured by Proposition~1.  
\end{proof}

\begin{proof}[Proof of Prop.~\ref{diameter_de}]
	Let us consider two sets $S=\{\xx_1,\dots, \xx_p\}, S'=\{\xx_1',\dots, \xx_p'\} \in \sset_p$. Then, from the fact that a correlation kernel is upper-bounded by $1$, we get
	\begin{align*}
	d_{\embr}^2(S,S') & = \frac{1}{p^2}\left(\sum_{i=1}^p\sum_{j=1}^p \rin(\xx_i,\xx_j)+\sum_{i=1}^p\sum_{j=1}^p \rin(\xx_i',\xx_j') \right.
	\\ 
	&\left. -2\sum_{i=1}^p\sum_{j=1}^p \rin(\xx_i,\xx_j')\right)\\
	& \leq \frac{1}{p^2}\left(2p^2 - 2\sum_{i=1}^p\sum_{j=1}^p \rin(\xx_i,\xx_j')\right)\\
	&\leq \frac{1}{p^2}\left(2p^2 - 2\sum_{i=1}^p\sum_{j=1}^p \rin(\mathbf{0}_d,\mathbf{1}_d)\right),
	\end{align*}
	where the last inequality follows from the assumed monotonicity of $\rin$ with respect to the Euclidean distance between elements of $\bset$ and the fact that the maximal distance between two points of $\bset$, i.e. the Euclidean diameter of $[0,1]^d$, is precisely attained for $\xx=\mathbf{0}_d$ and $\xx'=\mathbf{1}_d$. Finally, by assumption again, $\rin(\mathbf{0}_d,\mathbf{1}_d)$ is monotonically decreasing to $0$ when $\thetain$ decreases to $0$, and so the upper bound of $d_{\embr}^2$ tends to $\frac{1}{p^2}\left(2p^2 -0\right)=2$, showing that upper bound of the $d_{\embr}$-diameter of $\sset_p$ with respect to $\thetain \in (0,+\infty)$ is $\sqrt{2}$ indeed, independently of the dimension.  
\end{proof}

\section{Complements on the methodology}

\subsection{Maximum likelihood estimation for GPs with Deep Embedding kernel}
In the numerical experiments, we make predictions under a stationary GP model which assumes a constant unknown trend (following the route of Ordinary Kriging prediction such as exposed in \citep{roustant2012dicekriging}). When both $\kin$ and $\kout$ are assumed to be Gaussian kernels (still with the parametrization mentioned in \citep{roustant2012dicekriging}), the introduced Deep Embedding kernel takes the form
\begin{align}
k_{DE}(S,S') & = \kout \circ d_{\emb} (S,S')\notag\\
& = \sigout^2\rout\circ d_{\emb} (S,S')\label{kGG}\\
& = \sigout^2\exp\left(-\frac{1}{2}\frac{d^2_{\emb} (S,S')}{\thetaout^2}\right),
\end{align} 
where 
\begin{equation}
\begin{split}
d_{\emb}(S,S')& =\left(
\frac{1}{\#S\#S}\sum_{\xx_1,\xx_2\in S} \exp\left(-\frac{1}{2}\frac{\norm{\xx_1-\xx_2}^2}{\thetain^2}\right) \right.\\
&+
\frac{1}{\# S'\#S'}\sum_{\xx_1',\xx_2'\in S'} \exp\left(-\frac{1}{2}\frac{\norm{\xx'_1-\xx'_2}^2}{\thetain^2}\right)  \\ 
&-\left.
\frac{2}{\# S\# S'}\sum_{\xx\in S, \xx'\in S'}\exp\left(-\frac{1}{2}\frac{\norm{\xx-\xx'}^2}{\thetain^2}\right)
\right)^{\frac{1}{2}}.
\end{split}
\end{equation}

The three hyperparameters are determined by Maximum Likelihood Estimation (MLE). The expression of $k_{DE}$ as a function of $\rout$ in Equation \ref{kGG} allows us to use the concentrated log-likelihood, optimized with respect to $\thetaout$ and $\thetain$ via genetic algorithm with derivatives \citep{mebane2011genetic}. This can be done in a similar manner to the method given in Appendix A of \cite{roustant2012dicekriging}. Assuming positive values for the hyperparameters, the derivatives of $\rout(\cdot,\cdot)$ with respect to the two hyperparameters $\thetaout$ and $\thetain$ exist and are respectively given by:

\begin{equation}
\frac{\partial \rout(S,S')}{\partial \thetaout}  
= \rout(S,S')\left(\frac{d_{\emb} (S,S')^2}{\thetaout^3}\right),
\end{equation}
and
\begin{equation}
\frac{\partial \rout(S,S')}{\partial\thetain} 
= -\frac{1}{2\thetaout^2}\rout(S,S')\frac{\partial d_{\emb} (S,S')^2}{\partial \thetain},
\end{equation}
where
\begin{align}
\frac{\partial d_{\emb} (S,S')^2}{\partial \thetain} & = 
\frac{1}{\# S^2}\sum_{\xx_1,\xx_2\in S} \exp\left(-\frac{1}{2}\frac{\norm{\xx_1-\xx_2}^2}{\thetain^2}\right)\left(\frac{\norm{\xx_1-\xx_2}^2}{\thetain^3}\right)\notag
\\
&+\frac{1}{\# S'^2}\sum_{\xx_1',\xx_2'\in S'} \exp\left(-\frac{1}{2}\frac{\norm{\xx'_1-\xx'_2}^2}{\thetain^2}\right)\left(\frac{\norm{\xx'_1-\xx'_2}^2}{\thetain^3}\right)  \\
&-
\frac{2}{\# S\# S'}\sum_{\xx\in S, \xx'\in S'}\exp\left(-\frac{1}{2}\frac{\norm{\xx-\xx'}^2}{\thetain^2}\right)\left(\frac{\norm{\xx-\xx'}^2}{\thetain^3}\right)\notag
.
\end{align} 

\subsection{Condition number and jitter for matrix inversion}
The condition number of an $n\times n$ positive definite matrix $\textbf{R}$ under the $2$-norm is defined by
\begin{equation}
\kappa(\textbf{R}) = \norm{\textbf{R}}_2\norm{\textbf{R}^{-1}}_2 = \frac{\lambda_n}{\lambda_1},
\end{equation}
where $\lambda_n$ and $\lambda_1$ are the largest and smallest positive eigenvalues of $\textbf{R}$, respectively. A matrix is said to be ill-conditioned when its condition number is larger than some prescribed threshold. 

Given an ill-conditioned matrix, one can perturb the matrix by adding a small ``jitter" $\delta$ to diagonal in order to decrease its condition number:
\begin{equation}
\textbf{R}_\delta = \textbf{R} + \delta \textbf{I},
\end{equation}
where $\textbf{I}$ denotes the identity matrix with appropriate dimension. The eigenvalues of the perturbed matrix $\textbf{R}_\delta$ become $\lambda_i + \delta$, $i=1,2,3,...,n$ where $\lambda_i$ is the $i$th smallest eigenvalue of the original matrix $\textbf{R}$.

In Gaussian Process modelling, it is not rare that the inversion of ill-conditioned covariance/correlation matrices constitutes a bottleneck, motivating to introduce a positive jitter $\delta$; yet, finding an appropriate value for such a $\delta$ is no straightforward task and too small a value might not fix the issue of near singularity while too big a value could cause over-regularization and result in a poor surrogate of the inverse.
One  approach is to consider the jitter as a model hyperparameter and estimate it, e.g., by MLE. However, implementing this method may end up introducing positive jitter values even the matrix itself is well-conditioned. Also, things can be challenging from the computational point of view when $\delta$ takes a variety of values in the course of hyperparameter optimization. 

\cite{ranjan2011computationally} proposed an alternative way by finding a lower bound of the jitter that can overcome the ill-condition issue while minimizing the over-smoothing. As proven in \citep{ranjan2011computationally}, the condition number $\kappa(\textbf{R}_\delta)$, setting a jitter level to 
\begin{equation}
\delta\left(a\right) = \frac{\lambda_n\left(\kappa(\textbf{R})-\exp(a)\right)}{\kappa(\textbf{R})(\exp(a)-1)},
\label{deltalb}
\end{equation}
will ensure that the condition number of $\textbf{R}_\delta$ remains below a prescribed value $\exp(a)$. 

\section{Complementary experimental results}

\subsection{DS kernel +jitter for contaminant source localization test cases}
Due to conditioning issues in combinatorial problems, the double sum kernel is not readily applicable for the contaminant source localization test case. We hence apply the described jitter trick in the case of GP prediction with DS kernel on this test case. In particular, to find an appropriately small jitter, we vary the value of $``a" = 1,2,3,...,7$ in Equation \ref{deltalb}, and compare both prediction and optimization performances of the modified DS kernel when the corresponding bound values for the jitter are used. 

In the numerical experiments, once the jitter $\delta$ is set, the correlation matrix $\mathbf{R}_\delta = \mathbf{R}+\delta$ is used in all computations. This includes not only the computation of predictive mean and variance, but also the log-likelihood as well as its partial derivatives with respect to hyperparameters.

\subsubsection{Prediction performance}

Table \ref{tab:predcontwithdbs} gives $Q^2$ values for GP models with the proposed DE kernel against DS ones with multiple values of ``$a$" on the four considered scenarios for the contamination test case (refer to Table 1 in the main article).

We can see from the table that small values of ``$a$", e.g. $a= 1$ and $2$, which corresponds to larger jitter levels, yield higher prediction errors. Here in fact, the DE kernel outperforms the DS kernels on all cases. 

\begin{center}
	\begin{table}[h]
		\tiny
		\centering
		\caption{$Q^2$ values for GP predictions on contamination test cases with DE versus DS kernels ($\kde$ versus $\kds$+j)}
		\label{tab:predcontwithdbs}
		\begin{tabular}{|c|c|c|c|c|c|c|c|c|c|}
			\hline
			$Q^2$ & Ratio & $k_{DE}$ & $k_{0}+j1$ & $k_{0}+j2$ & $k_{0}+j3$ & $k_{0}+j4$ & $k_{0}+j5$ & $k_{0}+j6$ & $k_{0}+j7$ \\ \hline
			& 20:80 & 0.7607 & 0.3177 & 0.5756 & 0.7117 & 0.7501 & 0.7437 & 0.7109 & 0.6568 \\ \cline{2-10} 
			& 50:50 & 0.9133 & 0.3557 & 0.6506 & 0.7970 & 0.8391 & 0.8445 & 0.8438 & 0.8424 \\ \cline{2-10} 
			\multirow{-3}{*}{Src A, Geo 1} & 80:20 & 0.9352 & 0.4060 & 0.6930 & 0.8326 & 0.8728 & 0.8804 & 0.8815 & 0.8818 \\ \hline
			& 20:80 & 0.7239 & 0.2393 & 0.4884 & 0.6399 & 0.7013 & 0.7130 & 0.7025 & 0.6584 \\ \cline{2-10} 
			& 50:50 & 0.8855 & 0.3557 & 0.6430 & 0.8001 & 0.8449 & 0.8485 & 0.8476 & 0.8460 \\ \cline{2-10} 
			\multirow{-3}{*}{Src A, Geo 2} & 80:20 & 0.9240 & 0.3352 & 0.6514 & 0.8206 & 0.8673 & 0.8729 & 0.8724 & 0.8719 \\ \hline
			& 20:80 & 0.7977 & 0.2946 & 0.5457 & 0.7087 & 0.7775 & 0.7901 & 0.7720 & 0.7354 \\ \cline{2-10} 
			& 50:50 & 0.9190 & 0.3302 & 0.6450 & 0.8152 & 0.8668 & 0.8746 & 0.8749 & 0.8743 \\ \cline{2-10} 
			\multirow{-3}{*}{Src B, Geo 1} & 80:20 & 0.9447 & 0.3878 & 0.6847 & 0.8369 & 0.8818 & 0.8904 & 0.8916 & 0.8918 \\ \hline
			& 20:80 & 0.8486 & 0.2930 & 0.5672 & 0.7434 & 0.8182 & 0.8389 & 0.8398 & 0.8338 \\ \cline{2-10} 
			& 50:50 & 0.9151 & 0.3904 & 0.6916 & 0.8465 & 0.8880 & 0.8944 & 0.8946 & 0.8941 \\ \cline{2-10} 
			\multirow{-3}{*}{Src B, Geo 2} & 80:20 & 0.9439 & 0.4922 & 0.7543 & 0.8862 & 0.9207 & 0.9252 & 0.9259 & 0.9258 \\ \hline
		\end{tabular}
	\end{table}
\end{center}
Figures \ref{res-S1G1-2080}-\ref{res-S2G2-8020} show residual analyses for both leave-one-out and out-sample validation errors over four contaminant test cases. Here, we present only results for $\kds$+j2 and $\kds$+j5 (corresponding to the case when ``$a$"$=2$ and ``$a$"$=5$, respectively) to give a compact yet representative illustration of compared performances against the DE kernel. 

As one can see,  assigning an inappropriate ``$a$" value can lead to very poor predictive results ($a=2$). The fact that using the exposed approach with jitter heavily relies on the value of ``$a$" confers a relative robustness advantage to strictly positive definite DE kernels as no jitter is needed. This comes of course at the price of an additional hyperparameter to be fitted, yet with an estimation that can be more conveniently conducted together with the estimation of the other hyperparameters. 

\begin{figure*}[h!]
	\centering
	\flushleft{(a) $\kde$}
	\begin{minipage}{0.5\textwidth}
		\includegraphics[width=1\textwidth]{Figures/Cross-Validation-CONT-S1G1-PRED-RKHS-prop-2080-seed-1.png}
	\end{minipage}%
	\begin{minipage}{0.5\textwidth}
		\includegraphics[width=1\textwidth]{Figures/Prediction-CONT-S1G1-PRED-RKHS-prop-2080-seed-1.png}
	\end{minipage}
	\flushleft{(b) $\kds$+j2}
	\begin{minipage}{0.5\textwidth}
		\includegraphics[width=1\textwidth]{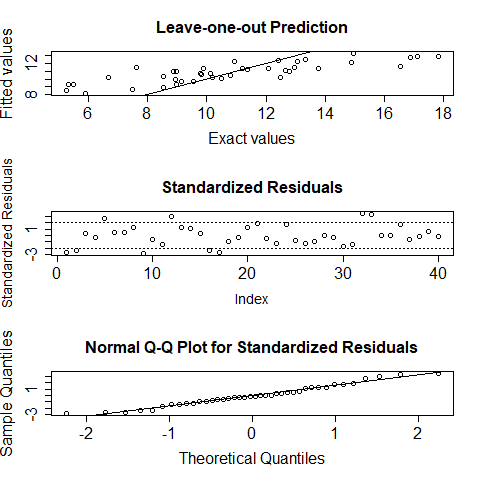}
	\end{minipage}%
	\begin{minipage}{0.5\textwidth}
		\includegraphics[width=1\textwidth]{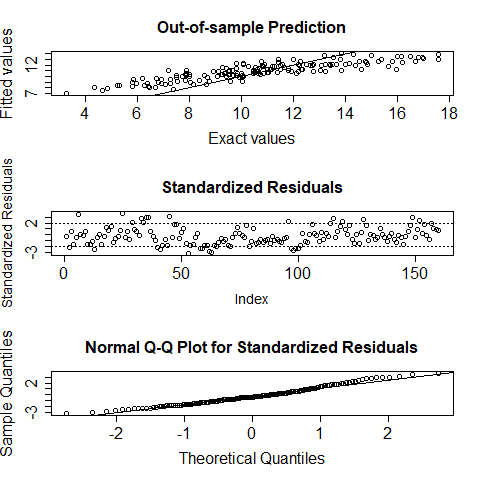}
	\end{minipage}
	\flushleft{(c) $\kds$+j5}
	\begin{minipage}{0.5\textwidth}
		\includegraphics[width=1\textwidth]{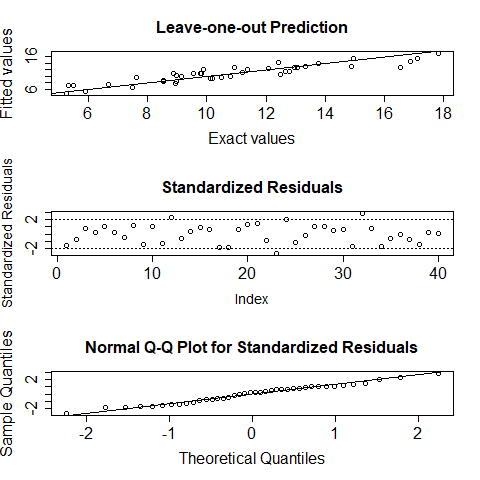}
	\end{minipage}%
	\begin{minipage}{0.5\textwidth}
		\includegraphics[width=1\textwidth]{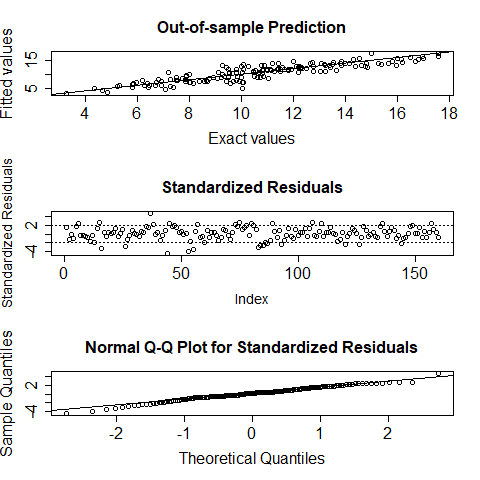}
	\end{minipage}
	\caption{Residual analysis on contamination test case \textbf{(Src A, Geo 1) with (20:80)}, (a) $\kde$, (b) $\kds$+j2 and (c) $\kds$+j5}
	\label{res-S1G1-2080}
\end{figure*}

\begin{figure*}[h!]
	\centering
	\flushleft{(a) $\kde$}
	\begin{minipage}{0.5\textwidth}
		\includegraphics[width=1\textwidth]{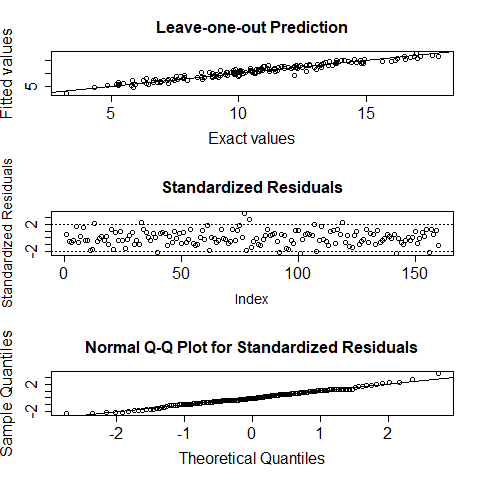}
	\end{minipage}%
	\begin{minipage}{0.5\textwidth}
		\includegraphics[width=1\textwidth]{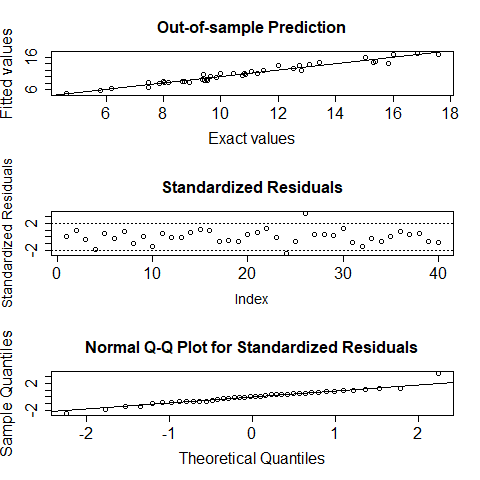}
	\end{minipage}
	\flushleft{(b) $\kds$+j2}
	\begin{minipage}{0.5\textwidth}
		\includegraphics[width=1\textwidth]{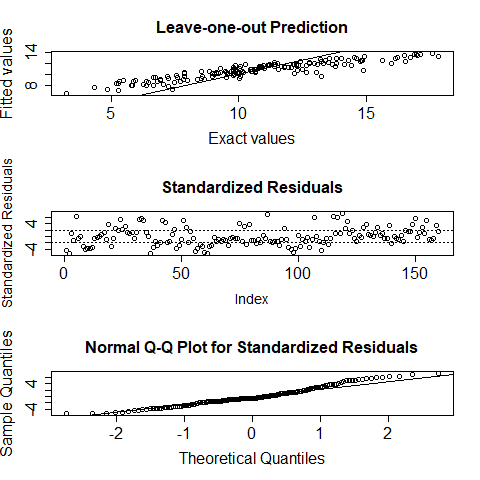}
	\end{minipage}%
	\begin{minipage}{0.5\textwidth}
		\includegraphics[width=1\textwidth]{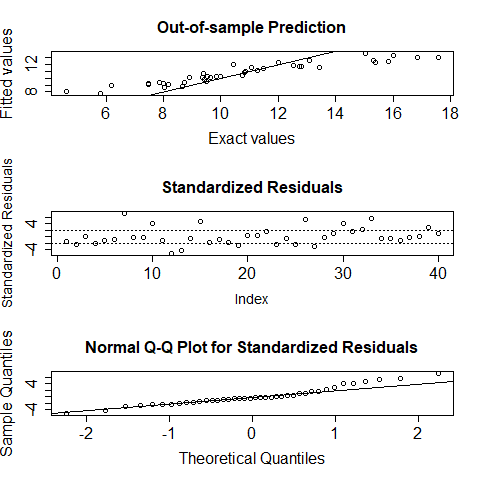}
	\end{minipage}
	\flushleft{(c) $\kds$+j5}
	\begin{minipage}{0.5\textwidth}
		\includegraphics[width=1\textwidth]{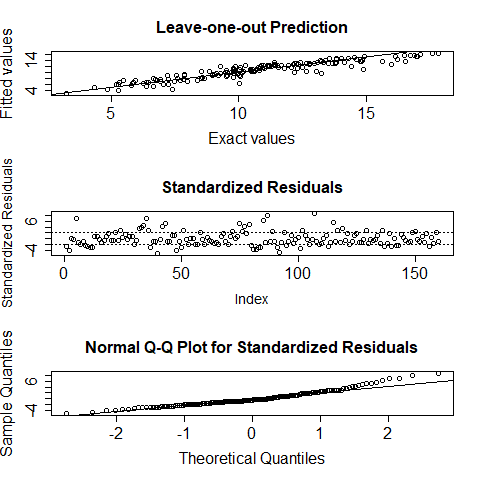}
	\end{minipage}%
	\begin{minipage}{0.5\textwidth}
		\includegraphics[width=1\textwidth]{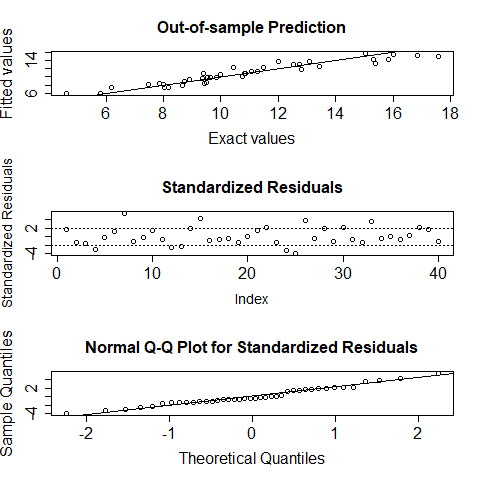}
	\end{minipage}
	\caption{Residual analysis on contamination test case \textbf{(Src A, Geo 1) with (80:20)}, (a) $\kde$, (b) $\kds$+j2 and (c) $\kds$+j5}
	\label{res-S1G1-8020}
\end{figure*}

\clearpage
\newpage

\begin{figure*}[h!]
	\centering
	\flushleft{(a) $\kde$}
	\begin{minipage}{0.5\textwidth}
		\includegraphics[width=1\textwidth]{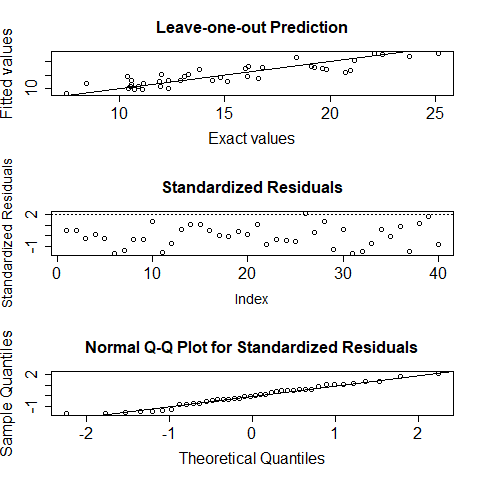}
	\end{minipage}%
	\begin{minipage}{0.5\textwidth}
		\includegraphics[width=1\textwidth]{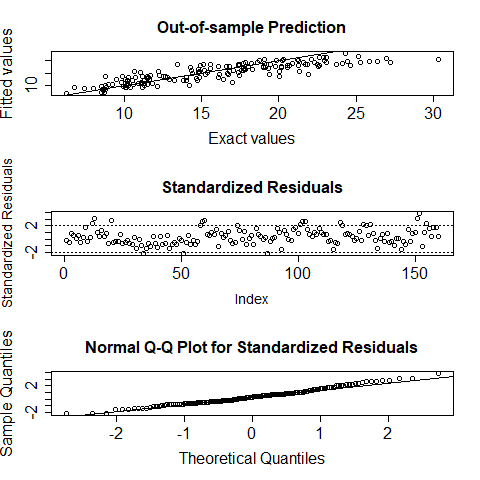}
	\end{minipage}
	\flushleft{(b) $\kds$+j2}
	\begin{minipage}{0.5\textwidth}
		\includegraphics[width=1\textwidth]{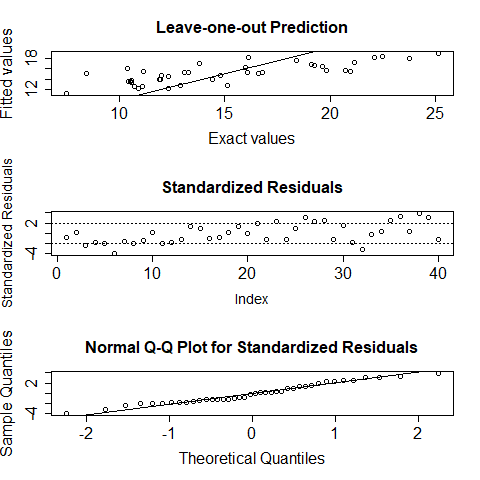}
	\end{minipage}%
	\begin{minipage}{0.5\textwidth}
		\includegraphics[width=1\textwidth]{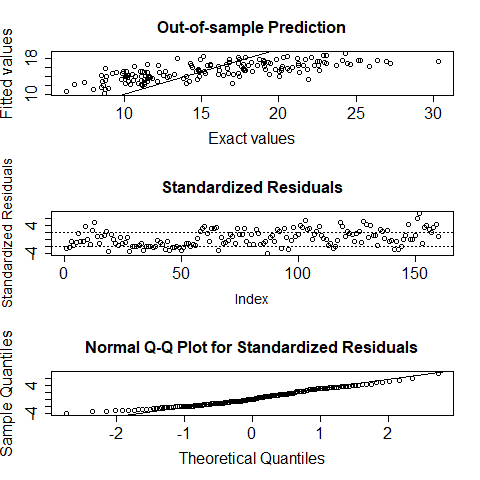}
	\end{minipage}
	\flushleft{(c) $\kds$+j5}
	\begin{minipage}{0.5\textwidth}
		\includegraphics[width=1\textwidth]{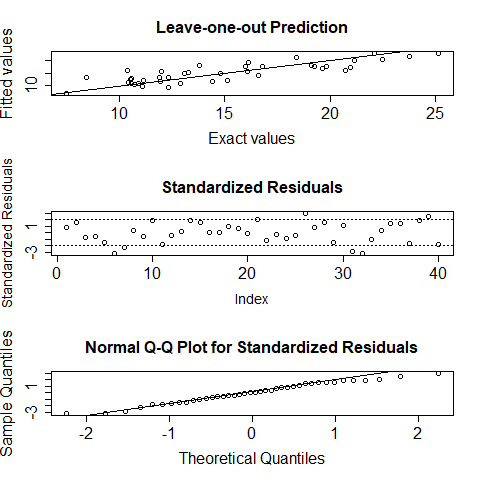}
	\end{minipage}%
	\begin{minipage}{0.5\textwidth}
		\includegraphics[width=1\textwidth]{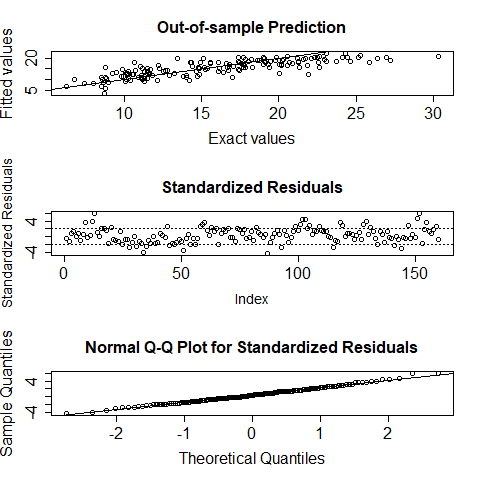}
	\end{minipage}
	\caption{Residual analysis on contamination test case \textbf{(Src A, Geo 2) with (20:80)}, (a) $\kde$, (b) $\kds$+j2 and (c) $\kds$+j5}
	\label{res-S1G2-2080}
\end{figure*}

\begin{figure*}[h!]
	\centering
	\flushleft{(a) $\kde$}
	\begin{minipage}{0.5\textwidth}
		\includegraphics[width=1\textwidth]{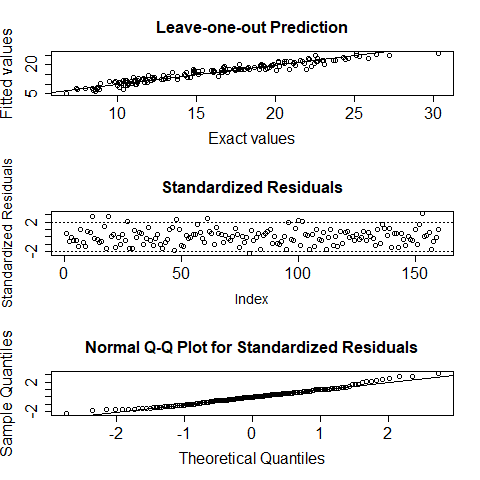}
	\end{minipage}%
	\begin{minipage}{0.5\textwidth}
		\includegraphics[width=1\textwidth]{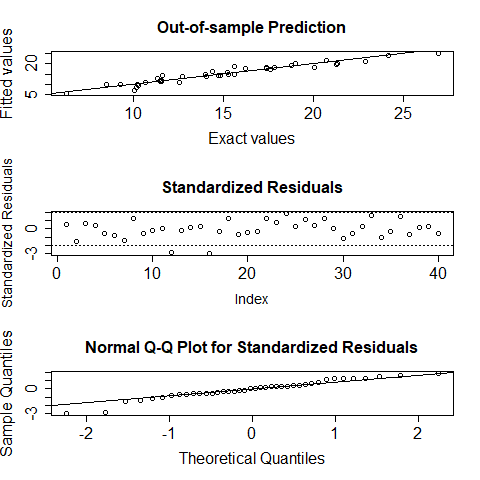}
	\end{minipage}
	\flushleft{(b) $\kds$+j2}
	\begin{minipage}{0.5\textwidth}
		\includegraphics[width=1\textwidth]{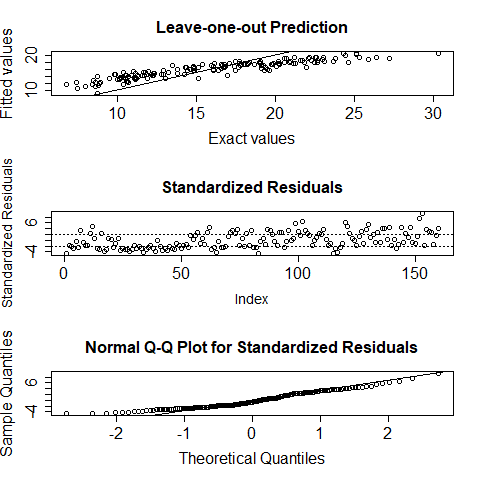}
	\end{minipage}%
	\begin{minipage}{0.5\textwidth}
		\includegraphics[width=1\textwidth]{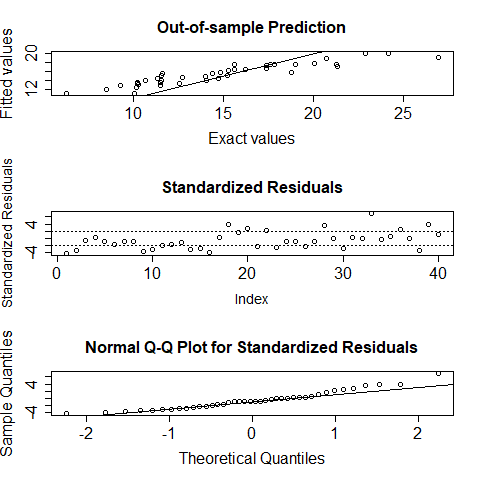}
	\end{minipage}
	\flushleft{(c) $\kds$+j5}
	\begin{minipage}{0.5\textwidth}
		\includegraphics[width=1\textwidth]{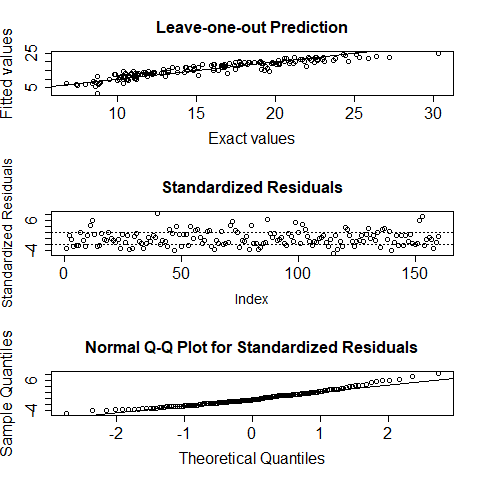}
	\end{minipage}%
	\begin{minipage}{0.5\textwidth}
		\includegraphics[width=1\textwidth]{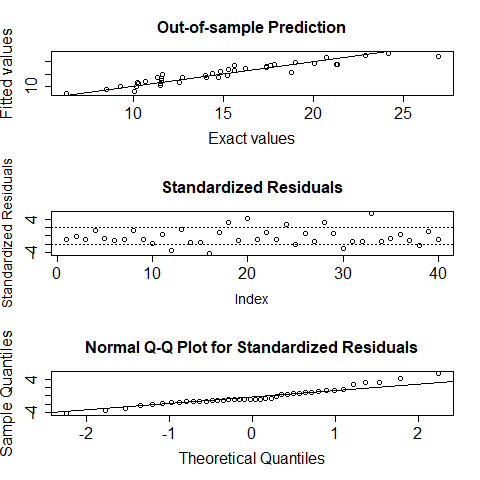}
	\end{minipage}
	\caption{Residual analysis on contamination test case \textbf{(Src A, Geo 2) with (80:20)}, (a) $\kde$, (b) $\kds$+j2 and (c) $\kds$+j5}
	\label{res-S1G2-8020}
\end{figure*}

\clearpage
\newpage

\begin{figure*}[h!]
	\centering
	\flushleft{(a) $\kde$}
	\begin{minipage}{0.5\textwidth}
		\includegraphics[width=1\textwidth]{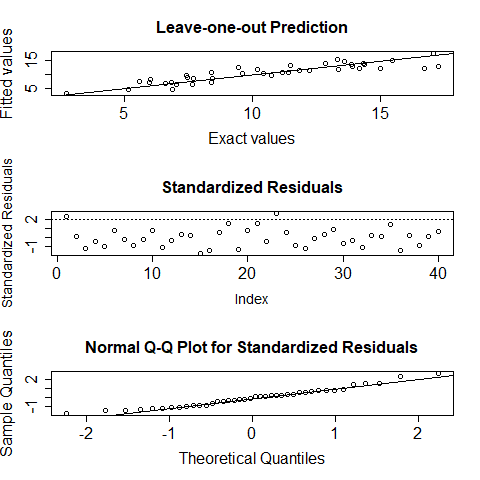}
	\end{minipage}%
	\begin{minipage}{0.5\textwidth}
		\includegraphics[width=1\textwidth]{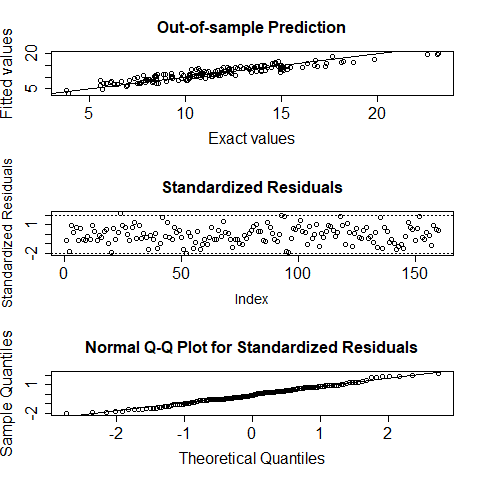}
	\end{minipage}
	\flushleft{(b) $\kds$+j2}
	\begin{minipage}{0.5\textwidth}
		\includegraphics[width=1\textwidth]{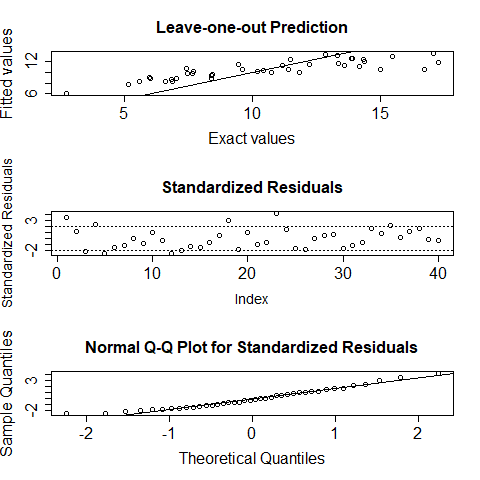}
	\end{minipage}%
	\begin{minipage}{0.5\textwidth}
		\includegraphics[width=1\textwidth]{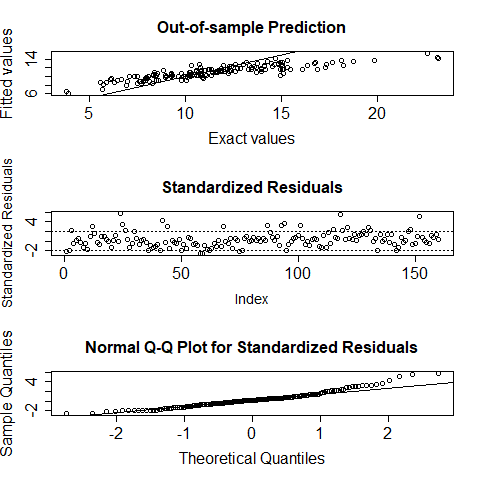}
	\end{minipage}
	\flushleft{(c) $\kds$+j5}
	\begin{minipage}{0.5\textwidth}
		\includegraphics[width=1\textwidth]{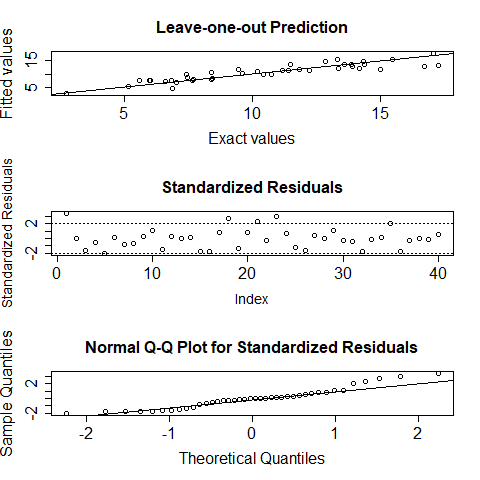}
	\end{minipage}%
	\begin{minipage}{0.5\textwidth}
		\includegraphics[width=1\textwidth]{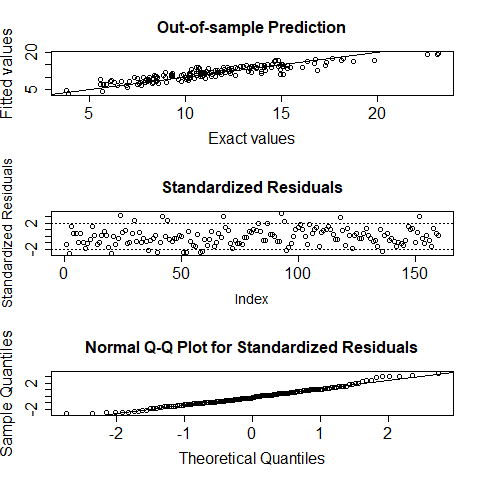}
	\end{minipage}
	\caption{Residual analysis on contamination test case \textbf{(Src B, Geo 1) with (20:80)}, (a) $\kde$, (b) $\kds$+j2 and (c) $\kds$+j5}
	\label{res-S2G1-2080}
\end{figure*}

\begin{figure*}[h!]
	\centering
	\flushleft{(a) $\kde$}
	\begin{minipage}{0.5\textwidth}
		\includegraphics[width=1\textwidth]{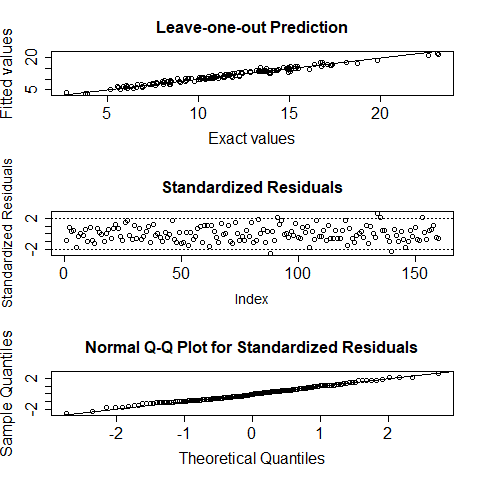}
	\end{minipage}%
	\begin{minipage}{0.5\textwidth}
		\includegraphics[width=1\textwidth]{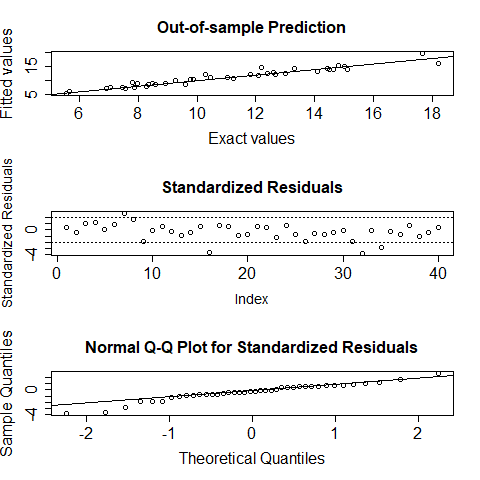}
	\end{minipage}
	\flushleft{(b) $\kds$+j2}
	\begin{minipage}{0.5\textwidth}
		\includegraphics[width=1\textwidth]{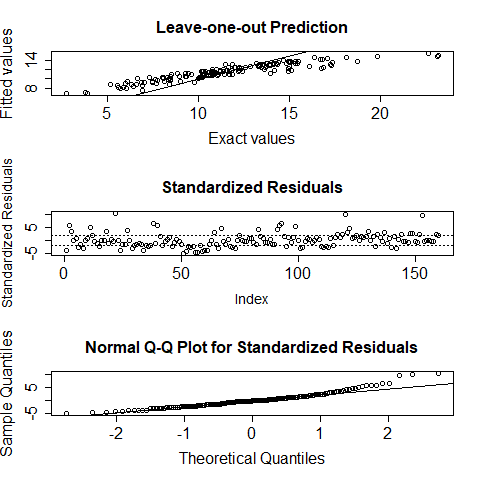}
	\end{minipage}%
	\begin{minipage}{0.5\textwidth}
		\includegraphics[width=1\textwidth]{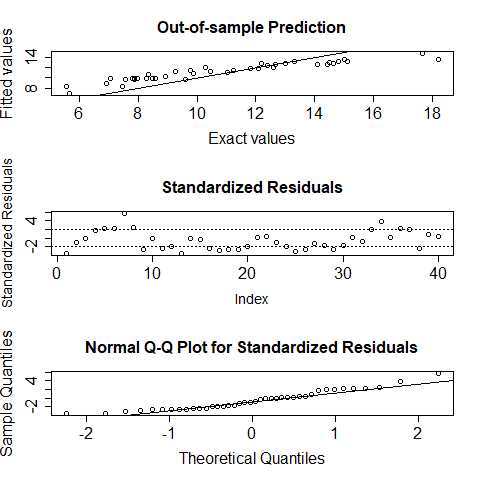}
	\end{minipage}
	\flushleft{(c) $\kds$+j5}
	\begin{minipage}{0.5\textwidth}
		\includegraphics[width=1\textwidth]{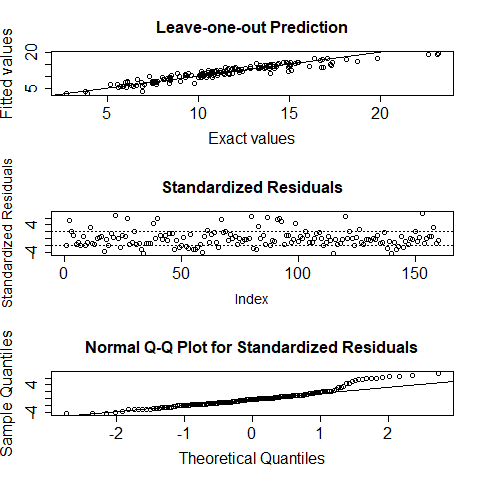}
	\end{minipage}%
	\begin{minipage}{0.5\textwidth}
		\includegraphics[width=1\textwidth]{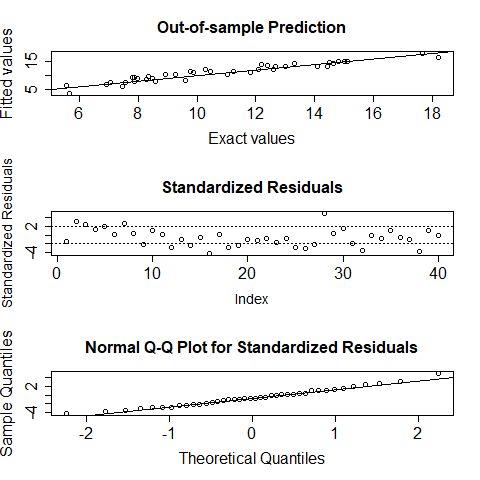}
	\end{minipage}
	\caption{Residual analysis on contamination test case \textbf{(Src B, Geo 1) with (80:20)}, (a) $\kde$, (b) $\kds$+j2 and (c) $\kds$+j5}
	\label{res-S2G1-8020}
\end{figure*}

\clearpage
\newpage

\begin{figure*}[h!]
	\centering
	\flushleft{(a) $\kde$}
	\begin{minipage}{0.5\textwidth}
		\includegraphics[width=1\textwidth]{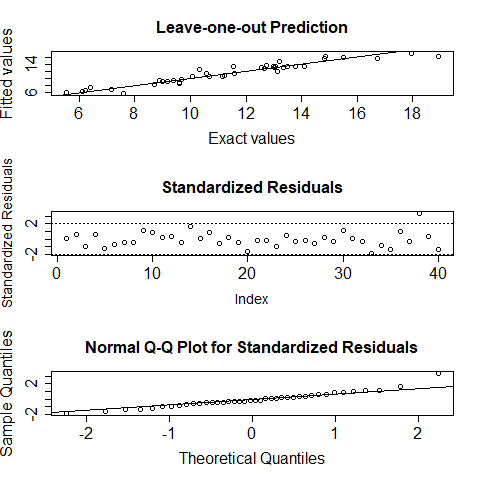}
	\end{minipage}%
	\begin{minipage}{0.5\textwidth}
		\includegraphics[width=1\textwidth]{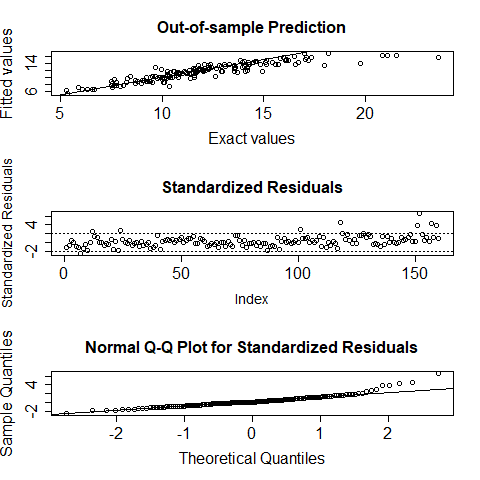}
	\end{minipage}
	\flushleft{(b) $\kds$+j2}
	\begin{minipage}{0.5\textwidth}
		\includegraphics[width=1\textwidth]{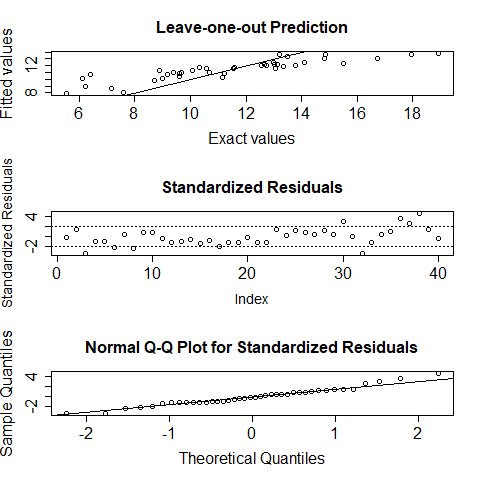}
	\end{minipage}%
	\begin{minipage}{0.5\textwidth}
		\includegraphics[width=1\textwidth]{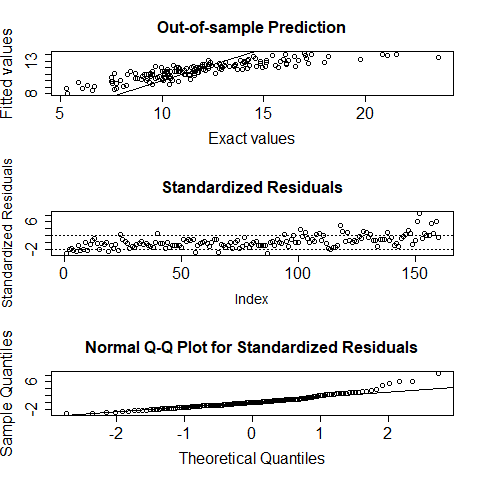}
	\end{minipage}
	\flushleft{(c) $\kds$+j5}
	\begin{minipage}{0.5\textwidth}
		\includegraphics[width=1\textwidth]{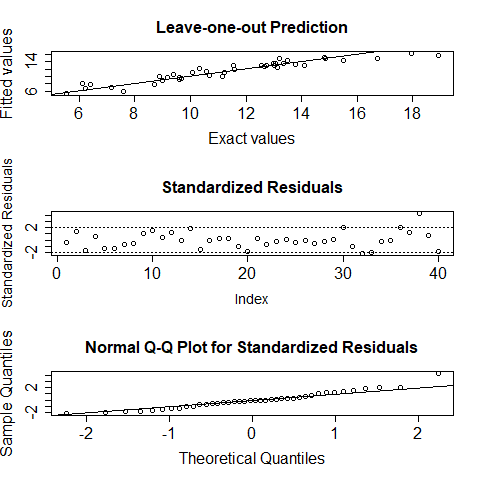}
	\end{minipage}%
	\begin{minipage}{0.5\textwidth}
		\includegraphics[width=1\textwidth]{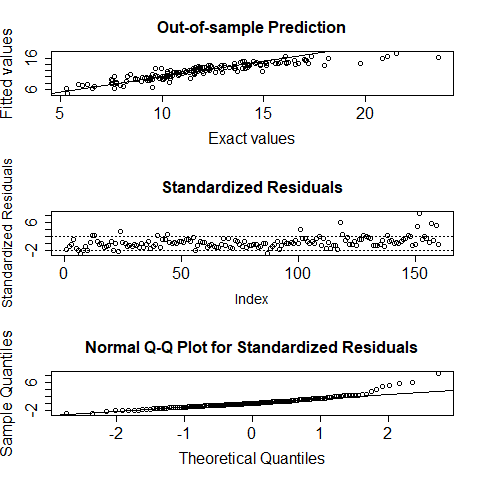}
	\end{minipage}
	\caption{Residual analysis on contamination test case \textbf{(Src B, Geo 2) with (20:80)}, (a) $\kde$, (b) $\kds$+j2 and (c) $\kds$+j5}
	\label{res-S2G2-2080}
\end{figure*}

\begin{figure*}[h!]
	\centering
	\flushleft{(a) $\kde$}
	\begin{minipage}{0.5\textwidth}
		\includegraphics[width=1\textwidth]{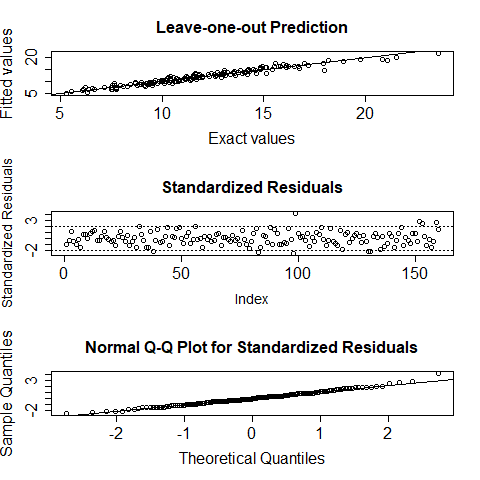}
	\end{minipage}%
	\begin{minipage}{0.5\textwidth}
		\includegraphics[width=1\textwidth]{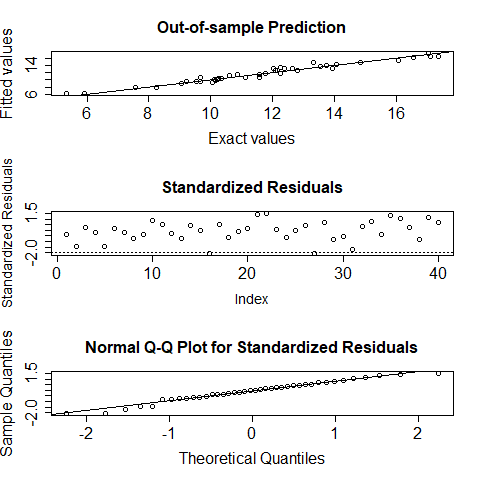}
	\end{minipage}
	\flushleft{(b) $\kds$+j2}
	\begin{minipage}{0.5\textwidth}
		\includegraphics[width=1\textwidth]{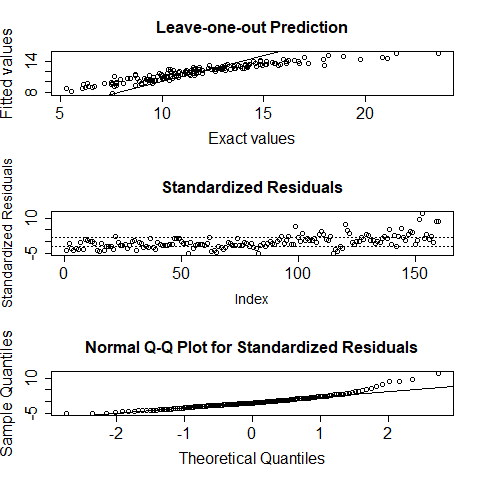}
	\end{minipage}%
	\begin{minipage}{0.5\textwidth}
		\includegraphics[width=1\textwidth]{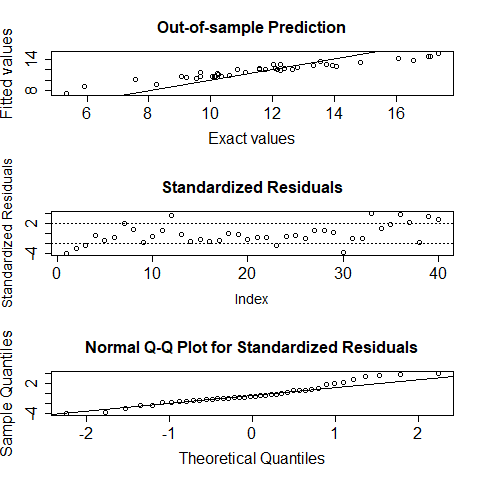}
	\end{minipage}
	\flushleft{(c) $\kds$+j5}
	\begin{minipage}{0.5\textwidth}
		\includegraphics[width=1\textwidth]{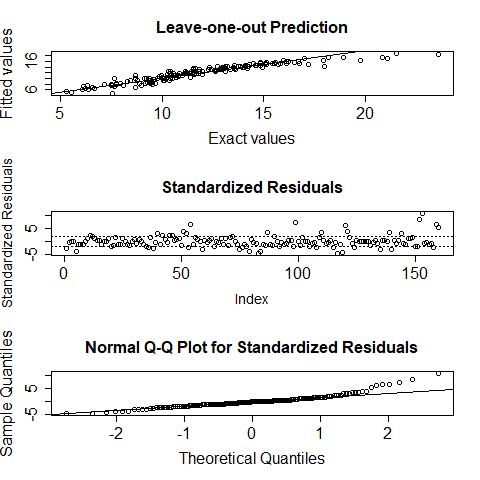}
	\end{minipage}%
	\begin{minipage}{0.5\textwidth}
		\includegraphics[width=1\textwidth]{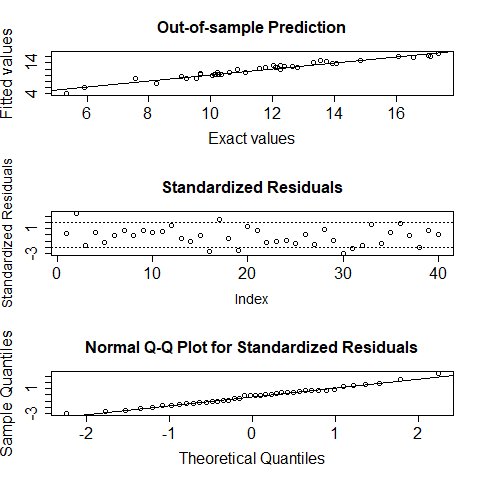}
	\end{minipage}
	\caption{Residual analysis on contamination test case \textbf{(Src B, Geo 2) with (80:20)}, (a) $\kde$, (b) $\kds$+j2 and (c) $\kds$+j5}
	\label{res-S2G2-8020}
\end{figure*}

\clearpage
\newpage
\subsubsection{Optimization performance}

In line with Section 3.3 of the main article, in this section, we present complete results of (1) the number of trials such that the minimum is found by EI with $\kde$ and $\kds+j$ in Table \ref{tab:bocontwithdbs}; (2) the  progress curves in terms of the median value of current best response in Figure \ref{fig:BO-median}; and (3) the 95th percentile of current best response in Figure \ref{fig:BO-per95}.

\begin{table}[h]
	\footnotesize
	\centering
	\caption{Number of trials (out of 100) such that minimum is found by EI algorithms with DE and DS kernels ($\kde$ versus $\kds$+j) on four contamination problems}
	\label{tab:bocontwithdbs}
	\begin{tabular}{|l|c|c|c|c|c|}
		\hline
		\textbf{Problem} & \textbf{EI-$k_{DE}$} & \textbf{EI-$k_{0}+j1$} & \textbf{EI-$k_{0}+j2$} & \textbf{EI-$k_{0}+j3$} & \textbf{EI-$k_{0}+j4$} \\ \hline
		\textbf{(a) Src A, Geo 1} & 100 & 17 & 63 & 87 & 95 \\ \hline
		\textbf{(b) Src A, Geo 2} & 66 & 15 & 36 & 46 & 52 \\ \hline
		\textbf{(c) Src B, Geo 1} & 100 & 26 & 59 & 77 & 95 \\ \hline
		\textbf{(d) Src B, Geo 2} & 78 & 42 & 64 & 76 & 81 \\ \hline
		\textbf{Problem} & \textbf{EI-$k_{0}+j5$} & \textbf{EI-$k_{0}+j6$} & \textbf{EI-$k_{0}+j7$} & \multicolumn{2}{c|}{\textbf{RANDOM}} \\ \hline
		\textbf{(a) Src A, Geo 1} & 98 & 96 & 97 & \multicolumn{2}{c|}{0} \\ \hline
		\textbf{(b) Src A, Geo 2} & 46 & 47 & 44 & \multicolumn{2}{c|}{0} \\ \hline
		\textbf{(c) Src B, Geo 1} & 96 & 96 & 95 & \multicolumn{2}{c|}{0} \\ \hline
		\textbf{(d) Src B, Geo 2} & 82 & 82 & 81 & \multicolumn{2}{c|}{0} \\ \hline
	\end{tabular}
\end{table}

\begin{figure}[h]
	\centering
	\begin{minipage}{0.25\textwidth}
		\includegraphics[width=1\textwidth]{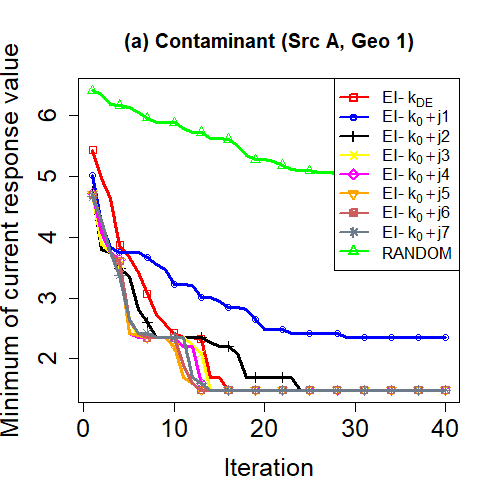}
	\end{minipage}%
	\begin{minipage}{0.25\textwidth}
		\includegraphics[width=1\textwidth]{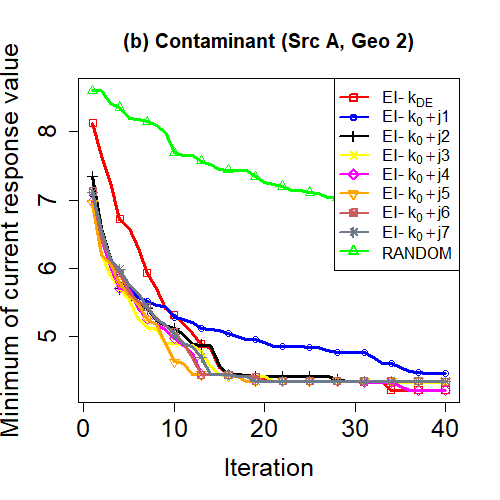}
	\end{minipage}%
	\begin{minipage}{0.25\textwidth}
		\includegraphics[width=1\textwidth]{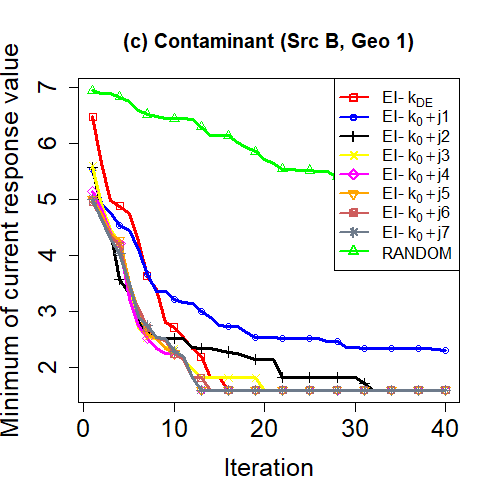}
	\end{minipage}%
	\begin{minipage}{0.25\textwidth}
		\includegraphics[width=1\textwidth]{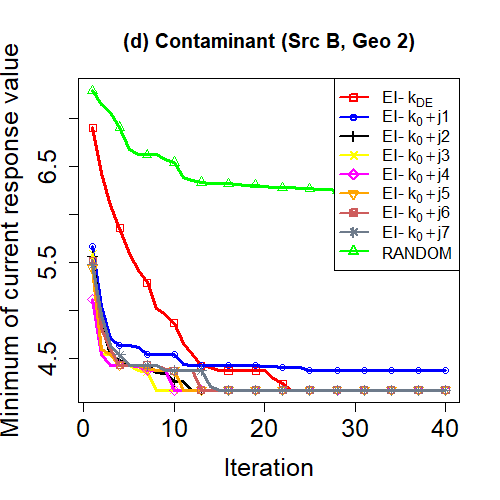}
	\end{minipage}%
	\caption{The median of current best response over 40 iterations on four contamination test cases} \label{fig:BO-median}
\end{figure}

\begin{figure}[h]
	\centering
	\begin{minipage}{0.25\textwidth}
		\includegraphics[width=1\textwidth]{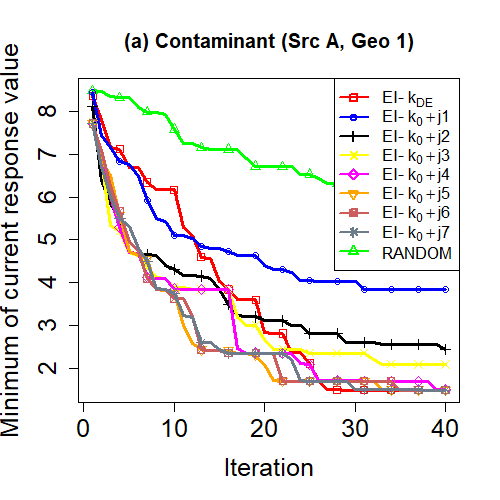}
	\end{minipage}%
	\begin{minipage}{0.25\textwidth}
		\includegraphics[width=1\textwidth]{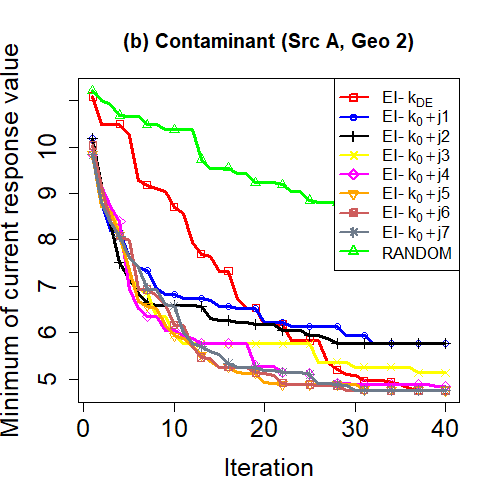}
	\end{minipage}%
	\begin{minipage}{0.25\textwidth}
		\includegraphics[width=1\textwidth]{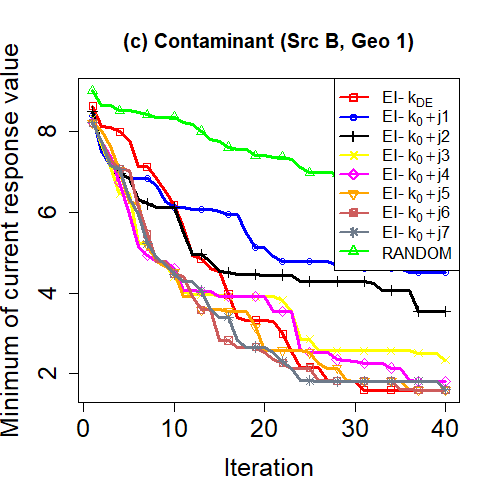}
	\end{minipage}%
	\begin{minipage}{0.25\textwidth}
		\includegraphics[width=1\textwidth]{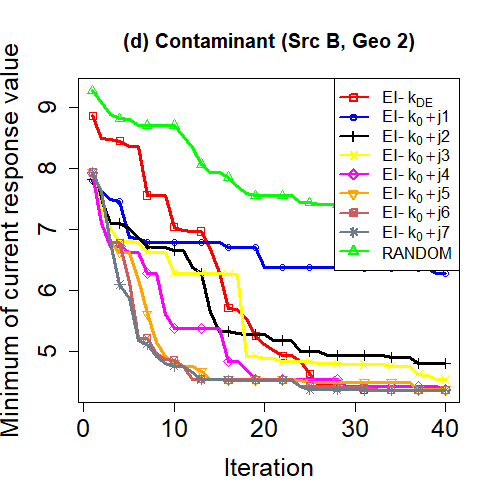}
	\end{minipage}%
	\caption{The 95th percentile of current best response over 40 iterations on four contamination test cases} \label{fig:BO-per95}
\end{figure}

Table \ref{tab:bocontwithdbs} indicates that with the DE kernel, EI could locate the true minimum for more replications than that with the DS kernels (at all jitter levels) for all problems, except for Source B, Geology 2. The progress curves of median and 95th percentile values suggest that regardless of the jitter level added, EI-$\kds+j$ method decreases the function value quickly at the beginning of the course when the kernel is still very well conditioned. With more points in the observation sets, jitter cannot be avoided as the kernel becomes ill-conditioned. When this happens, the performance of $\kds+j$ heavily depends on the jitter levels, as the progress curve starts to flatten out. Notice how the EI-$\kde$ curve crosses the EI-$\kds+j$ one in the 95th percentile plots. Because the model accuracy as well as optimization performance of the DS kernel relies on the jitter levels, this makes the approach less robust than the DE kernel.

\clearpage
\newpage
\subsection{Complementary residual analyses for the synthetic and Castem test cases}
\begin{figure*}[h!]
	\centering
	\flushleft{(a) $\kde$}
	\begin{minipage}{0.5\textwidth}
		\includegraphics[width=1\textwidth]{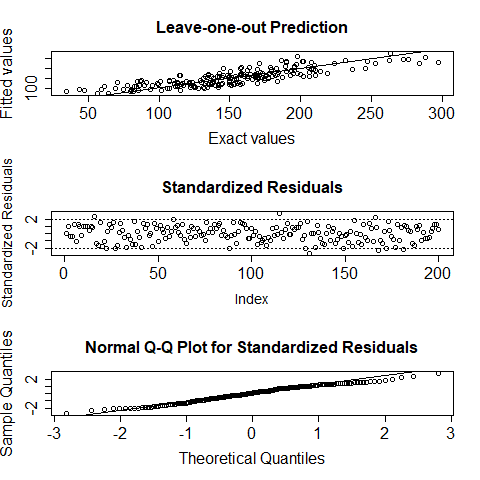}
	\end{minipage}%
	\begin{minipage}{0.5\textwidth}
		\includegraphics[width=1\textwidth]{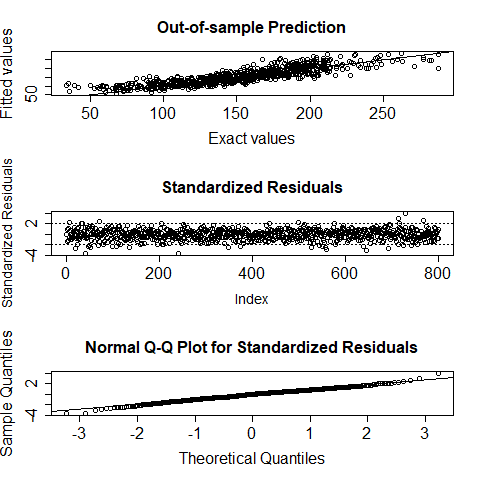}
	\end{minipage}
	\flushleft{(b) $\kds$}
	\begin{minipage}{0.5\textwidth}
		\includegraphics[width=1\textwidth]{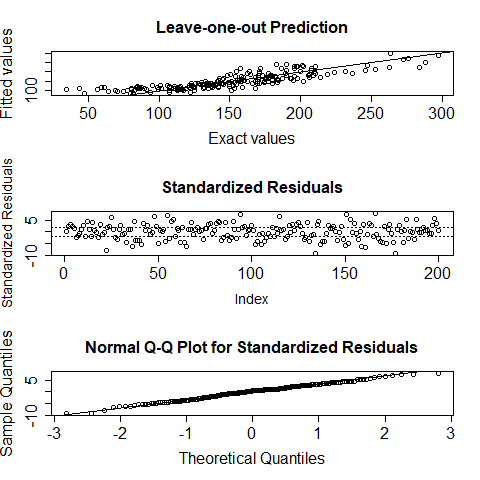}
	\end{minipage}%
	\begin{minipage}{0.5\textwidth}
		\includegraphics[width=1\textwidth]{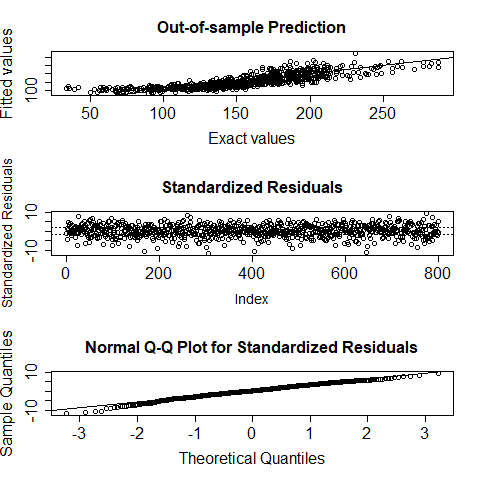}
	\end{minipage}
	\caption{Residual analysis on \textbf{MAX with (20:80)}, (a) $\kde$ and (b) $\kds$}
	\label{res-max-2080}
\end{figure*}

\begin{figure*}[h!]
	\centering
	\flushleft{(a) $\kde$}
	\begin{minipage}{0.5\textwidth}
		\includegraphics[width=1\textwidth]{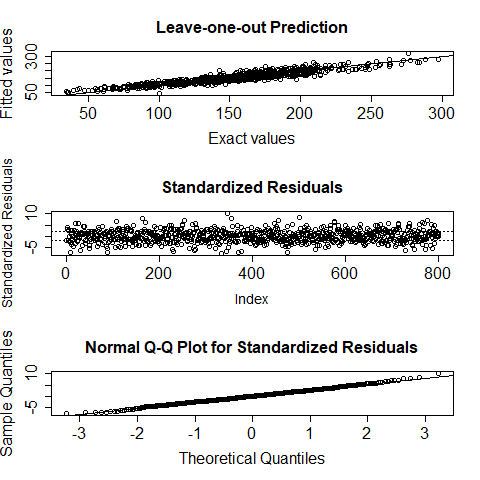}
	\end{minipage}%
	\begin{minipage}{0.5\textwidth}
		\includegraphics[width=1\textwidth]{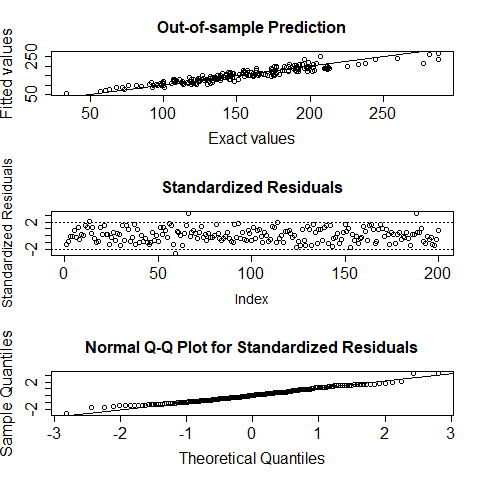}
	\end{minipage}
	\flushleft{(b) $\kds$}
	\begin{minipage}{0.5\textwidth}
		\includegraphics[width=1\textwidth]{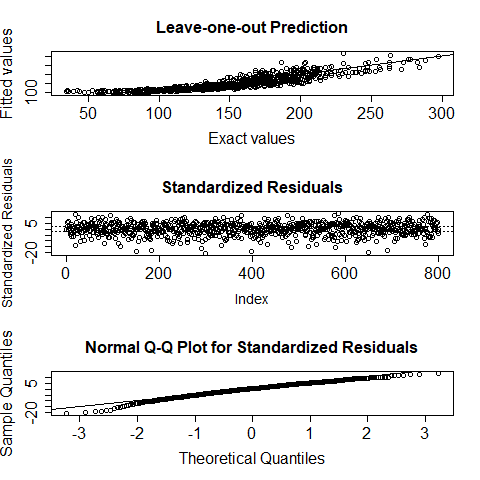}
	\end{minipage}%
	\begin{minipage}{0.5\textwidth}
		\includegraphics[width=1\textwidth]{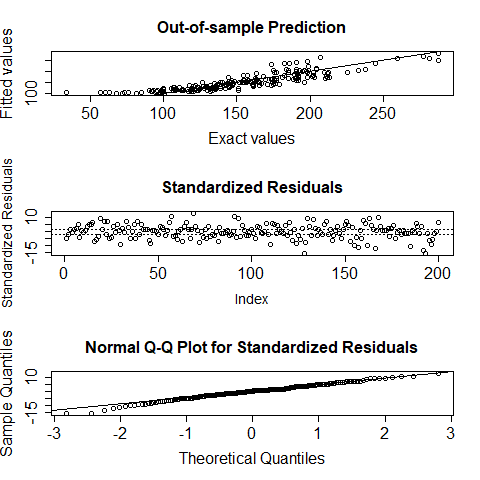}
	\end{minipage}
	\caption{Residual analysis on \textbf{MAX with (80:20)}, (a) $\kde$ and (b) $\kds$}
	\label{res-max-8020}
\end{figure*}

\clearpage
\newpage
\begin{figure*}[h!]
	\centering
	\flushleft{(a) $\kde$}
	\begin{minipage}{0.5\textwidth}
		\includegraphics[width=1\textwidth]{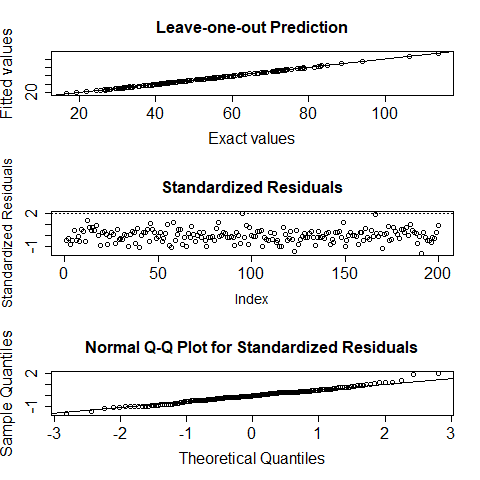}
	\end{minipage}%
	\begin{minipage}{0.5\textwidth}
		\includegraphics[width=1\textwidth]{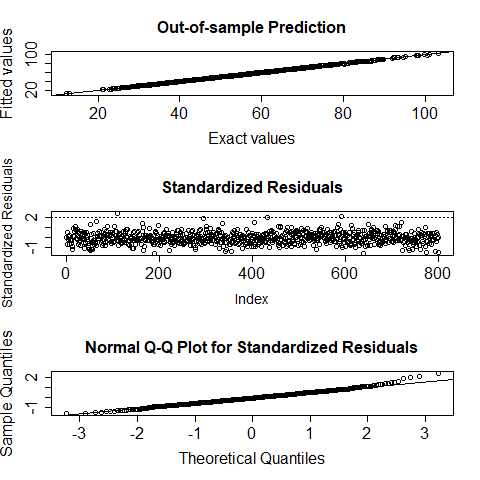}
	\end{minipage}
	\flushleft{(b) $\kds$}
	\begin{minipage}{0.5\textwidth}
		\includegraphics[width=1\textwidth]{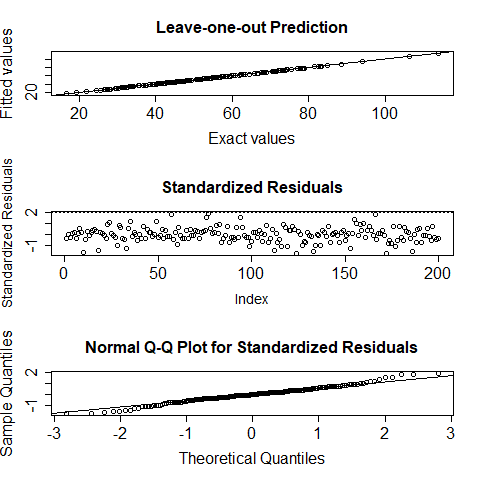}
	\end{minipage}%
	\begin{minipage}{0.5\textwidth}
		\includegraphics[width=1\textwidth]{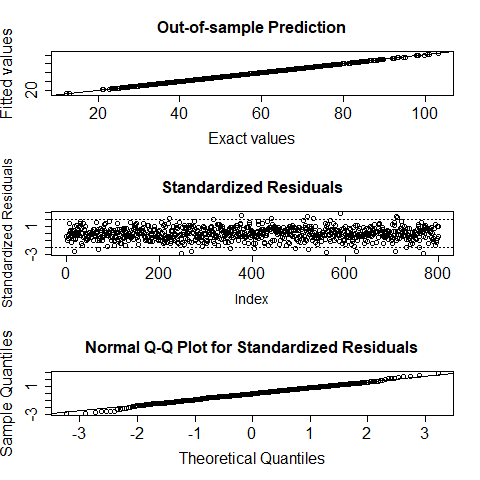}
	\end{minipage}
	\caption{Residual analysis on \textbf{MEAN with (20:80)}, (a) $\kde$ and (b) $\kds$}
	\label{res-mean-2080}
\end{figure*}

\begin{figure*}[h!]
	\centering
	\flushleft{(a) $\kde$}
	\begin{minipage}{0.5\textwidth}
		\includegraphics[width=1\textwidth]{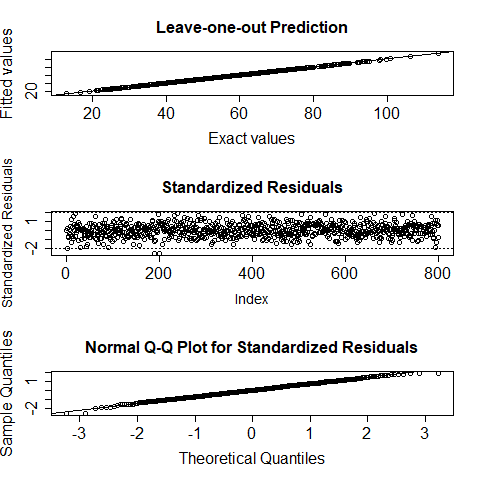}
	\end{minipage}%
	\begin{minipage}{0.5\textwidth}
		\includegraphics[width=1\textwidth]{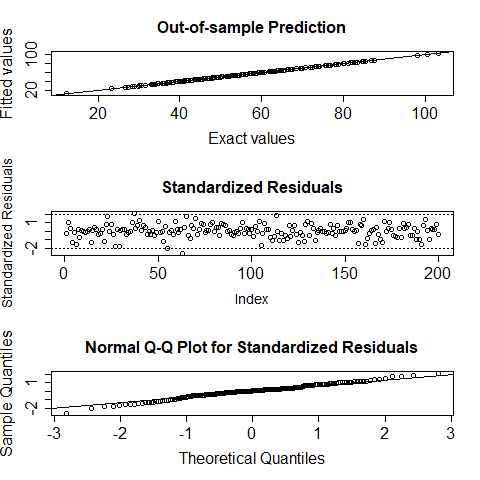}
	\end{minipage}
	\flushleft{(b) $\kds$}
	\begin{minipage}{0.5\textwidth}
		\includegraphics[width=1\textwidth]{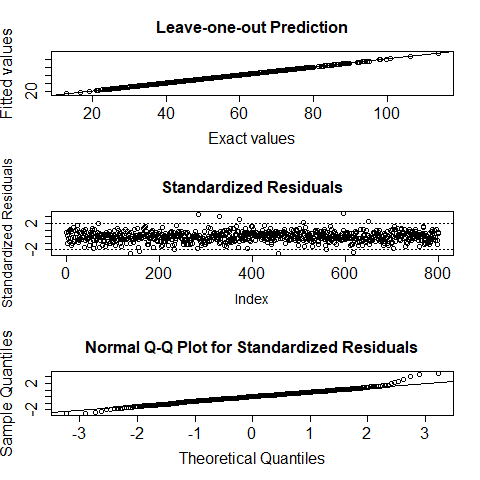}
	\end{minipage}%
	\begin{minipage}{0.5\textwidth}
		\includegraphics[width=1\textwidth]{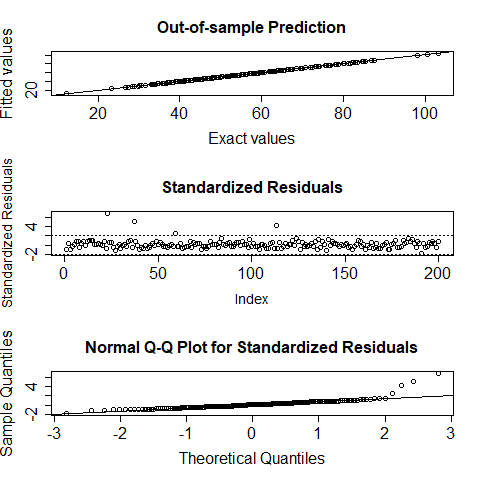}
	\end{minipage}
	\caption{Residual analysis on \textbf{MEAN with (80:20)}, (a) $\kde$ and (b) $\kds$}
	\label{res-mean-8020}
\end{figure*}

\clearpage
\newpage
\begin{figure*}[h!]
	\centering
	\flushleft{(a) $\kde$}
	\begin{minipage}{0.5\textwidth}
		\includegraphics[width=1\textwidth]{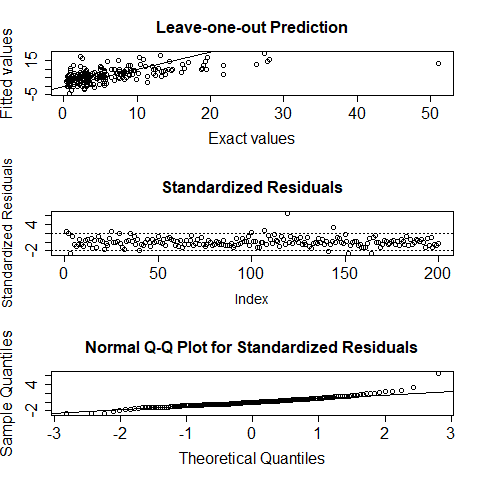}
	\end{minipage}%
	\begin{minipage}{0.5\textwidth}
		\includegraphics[width=1\textwidth]{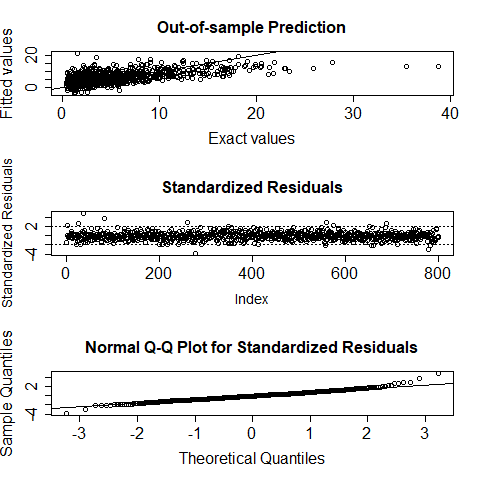}
	\end{minipage}
	\flushleft{(b) $\kds$}
	\begin{minipage}{0.5\textwidth}
		\includegraphics[width=1\textwidth]{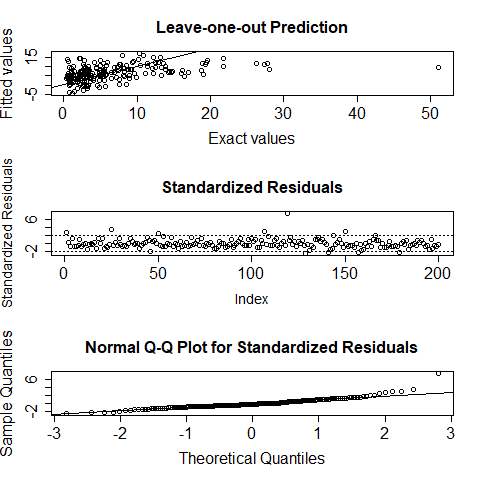}
	\end{minipage}%
	\begin{minipage}{0.5\textwidth}
		\includegraphics[width=1\textwidth]{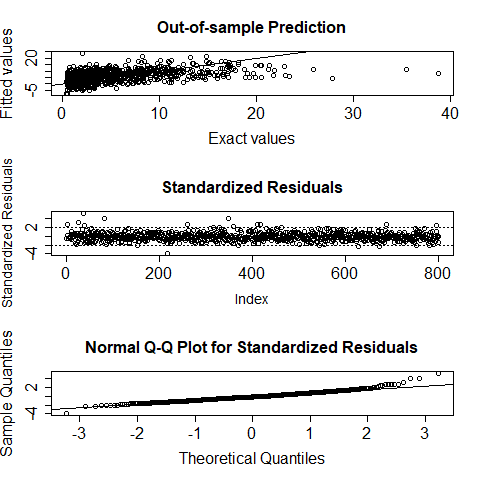}
	\end{minipage}
	\caption{Residual analysis on \textbf{MIN with (20:80)}, (a) $\kde$ and (b) $\kds$}
	\label{res-min-2080}
\end{figure*}

\begin{figure*}[h!]
	\centering
	\flushleft{(a) $\kde$}
	\begin{minipage}{0.5\textwidth}
		\includegraphics[width=1\textwidth]{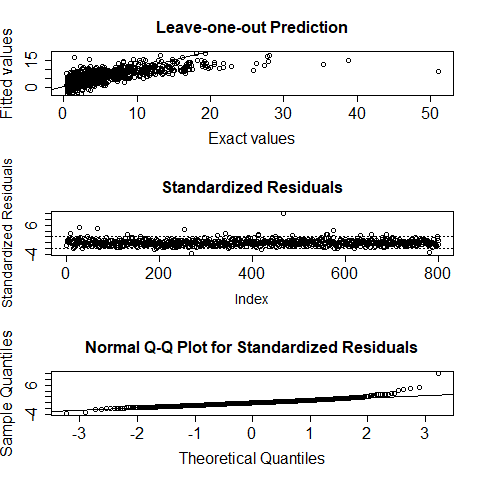}
	\end{minipage}%
	\begin{minipage}{0.5\textwidth}
		\includegraphics[width=1\textwidth]{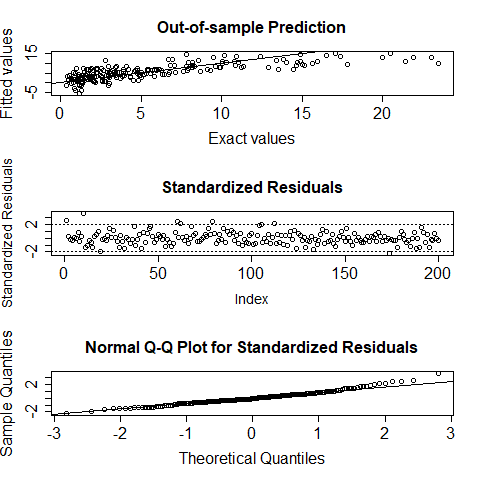}
	\end{minipage}
	\flushleft{(b) $\kds$}
	\begin{minipage}{0.5\textwidth}
		\includegraphics[width=1\textwidth]{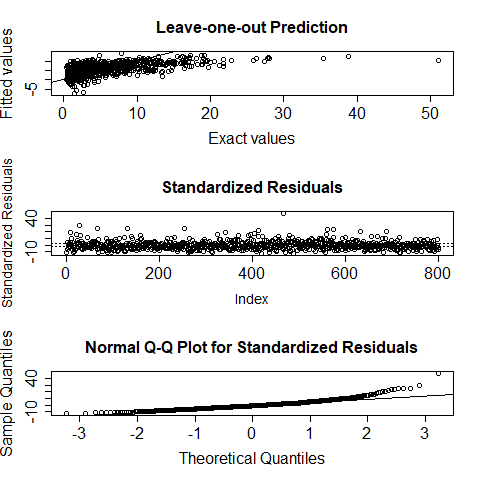}
	\end{minipage}%
	\begin{minipage}{0.5\textwidth}
		\includegraphics[width=1\textwidth]{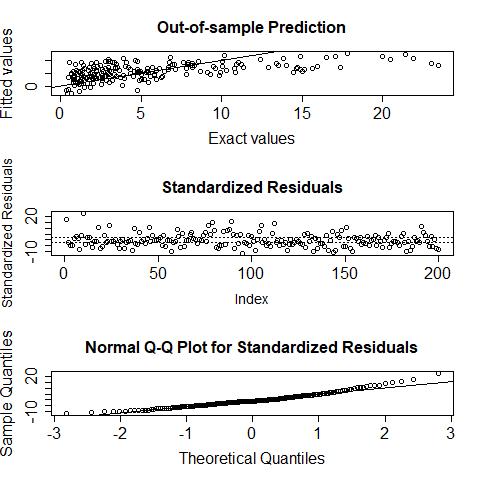}
	\end{minipage}
	\caption{Residual analysis on \textbf{MIN with (80:20)}, (a) $\kde$ and (b) $\kds$}
	\label{res-min-8020}
\end{figure*}

\clearpage
\newpage
\begin{figure*}[h!]
	\centering
	\flushleft{(a) $\kde$}
	\begin{minipage}{0.5\textwidth}
		\includegraphics[width=1\textwidth]{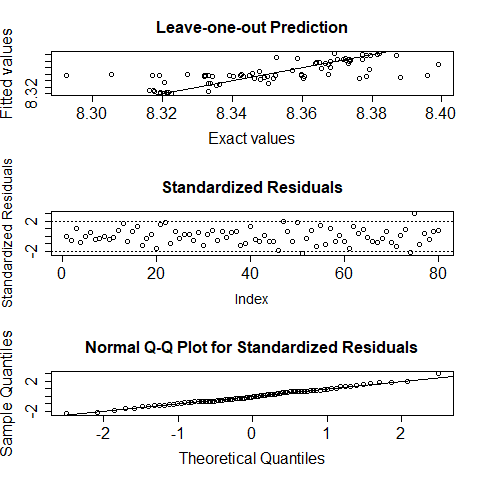}
	\end{minipage}%
	\begin{minipage}{0.5\textwidth}
		\includegraphics[width=1\textwidth]{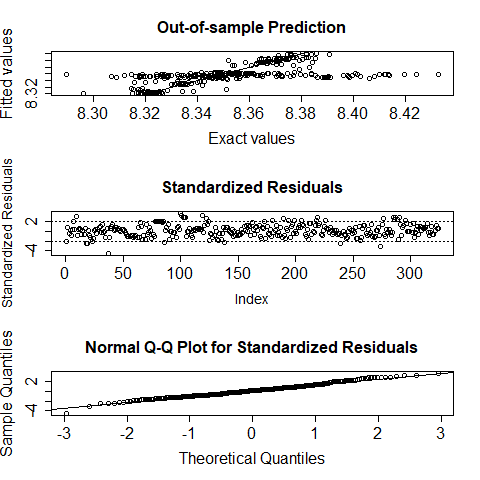}
	\end{minipage}
	\flushleft{(b) $\kds$}
	\begin{minipage}{0.5\textwidth}
		\includegraphics[width=1\textwidth]{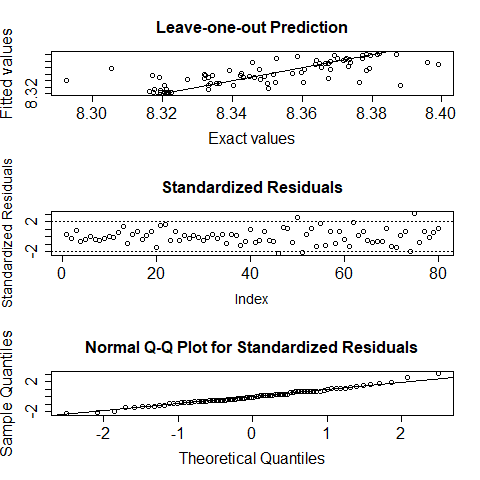}
	\end{minipage}%
	\begin{minipage}{0.5\textwidth}
		\includegraphics[width=1\textwidth]{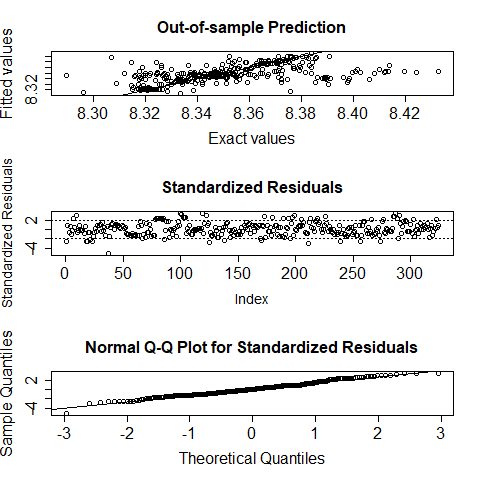}
	\end{minipage}
	\caption{Residual analysis on \textbf{CASTEM with (20:80)}, (a) $\kde$ and (b) $\kds$}
	\label{res-castem-2080}
\end{figure*}

\begin{figure*}[h!]
	\centering
	\flushleft{(a) $\kde$}
	\begin{minipage}{0.5\textwidth}
		\includegraphics[width=1\textwidth]{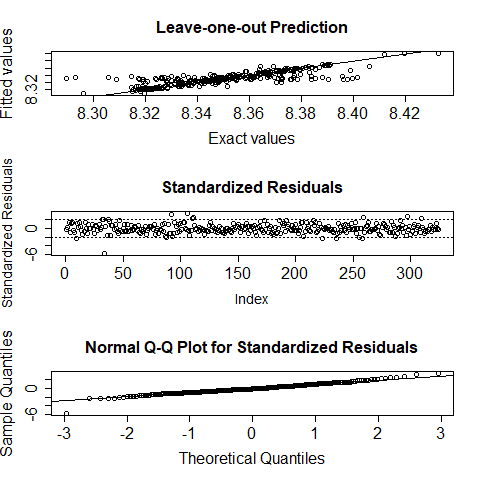}
	\end{minipage}%
	\begin{minipage}{0.5\textwidth}
		\includegraphics[width=1\textwidth]{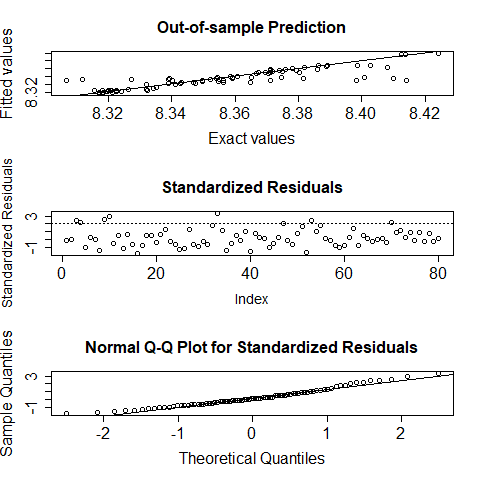}
	\end{minipage}
	\flushleft{(b) $\kds$}
	\begin{minipage}{0.5\textwidth}
		\includegraphics[width=1\textwidth]{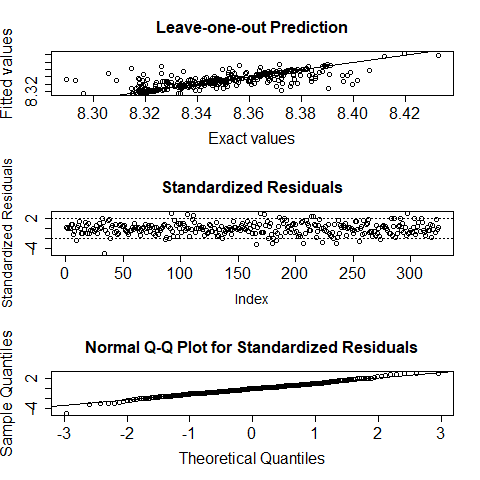}
	\end{minipage}%
	\begin{minipage}{0.5\textwidth}
		\includegraphics[width=1\textwidth]{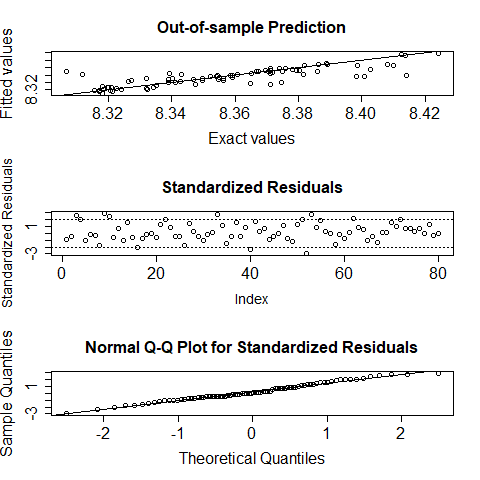}
	\end{minipage}
	\caption{Residual analysis on \textbf{CASTEM with (80:20)}, (a) $\kde$ and (b) $\kds$}
	\label{res-castem-8020}
\end{figure*}

\end{document}